\documentclass{amsart}
\usepackage{fullpage}
\usepackage[a4paper,margin=1in]{geometry}
\usepackage{style}

\usepackage[utf8]{inputenc} 
\usepackage[T1]{fontenc}    
\usepackage{url}            
\usepackage{booktabs}       
\usepackage{nicefrac}       
\usepackage{microtype}      

\usepackage[numbers]{natbib}

\usepackage{tikz}
\usetikzlibrary{arrows.meta,decorations.pathmorphing,positioning,calc,patterns}
\hypersetup{hypertexnames=false}

\newcommand{\Id}{\mathrm{Id}}

\renewcommand{\paragraph}[1]{%
  \par\smallskip
  \noindent\textbf{#1}\hspace{1em}\ignorespaces
}

\title{Propagation of Chaos in Contextual Flow Maps}

\author[S. Chen]{Shi Chen}
\address{(SC) Department of Mathematics, Massachusetts Institute of Technology, 77 Massachusetts Ave, 02139 Cambridge MA, USA} 
\email{schen636@mit.edu}

\author[Z. Lin]{Zhengjiang Lin}
\address{(ZL) Department of Mathematics, Massachusetts Institute of Technology, 77 Massachusetts Ave, 02139 Cambridge MA, USA} 
\email{linzj@mit.edu}

\author[K. Liu]{Kaizhao Liu}
\address{(KL) Department of Mathematics, Massachusetts Institute of Technology, 77 Massachusetts Ave, 02139 Cambridge MA, USA} 
\email{mrzt@mit.edu}

\author[P. Rigollet]{Philippe Rigollet}
\address{(PR) Department of Mathematics, Massachusetts Institute of Technology, 77 Massachusetts Ave, 02139 Cambridge MA, USA} 
\email{rigollet@math.mit.edu}

\begin{document}

\begin{abstract}
We develop a quantitative statistical theory of transformers in the long-context regime by adopting the abstraction of \emph{contextual flow maps} (CFMs): dynamical systems that evolve a distinguished token in the presence of a contextual measure across a stack of attention blocks. 
Within this framework, the finite-context model approximates an idealized infinite-context system in which the contextual measure is replaced by its underlying population, so that the context length~$n$ becomes a statistical resource. 
Exploiting the McKean--Vlasov structure of the dynamics and the classical machinery of propagation of chaos, we establish a \emph{forward} bound controlling the deviation between the finite- and infinite-context CFMs uniformly along depth, and a \emph{backward} bound controlling the deviation between the corresponding training trajectories uniformly across iterations of online gradient descent. 
Both bounds achieve the optimal Wasserstein rate~$n^{-1/d}$ for general CFMs and parametric rate $n^{-1/2}$ for a restricted class of CFMs that includes transformers as a special case. 
The analysis rests on a new Eulerian adjoint formulation of the loss gradient and stability estimates for the resulting forward--adjoint system, both of which may be of independent interest. 

\end{abstract}

\maketitle

\section{Introduction}

The transformer architecture, introduced by Vaswani et al.~\cite{vaswani2017attention}, is the foundation of modern machine learning systems, from large language models~\cite{devlin2019bert} to Vision Transformers~\cite{dosovitskiy2020image}. A convenient way to view inputs to these models is through a distinguished token $x \in \RR^d$ that is updated in the presence of a context of $n$ tokens $z^{(1)},\ldots,z^{(n)} \in \RR^d$ encoded by the \emph{contextual measure} $\mu=\frac{1}{n} \sum_i\delta_{z^{(i)}}$. In encoder-style models, $x$ is typically a prepended special token such as \texttt{[CLS]}; in decoder-style models, it is the current final token whose representation is used to predict the next one. After passing through a deep stack of attention blocks, the model returns an output token $y$ whose value depends on both $x$ and its context. We refer to such maps as \emph{contextual maps}.

This viewpoint on transformers was developed in~\cite{furuyaTransformersAreUniversal2024}, where transformers were shown to have universal approximation properties within the class of contextual maps. At the same time, the compositional structure of transformers, like that of deep neural networks more broadly, can be interpreted as a dynamical system, with the network realizing an input-to-output \emph{flow map}~\cite{E2017APO,haber2018stable,chen2018neural}. Combining these perspectives suggests a natural abstraction: transformers as \emph{contextual flow maps}, namely dynamical systems that evolve an input token-context pair $(x_0,\mu_0)$ across layers and produce an output token $x_1$. We consider a class of general contextual flow maps $(x_0, \mu_0) \mapsto x_1$ that include transformers and take the form
\begin{subequations}\label{eq:model intro}
\begin{empheq}[left=\empheqlbrace]{align}
&\dot{x}_s = \cV(x_s,\mu_s;\theta(s)), \quad s\in [0,1], \label{eq:token}\\
&\partial_s \mu_s + \nabla\cdot \bigl(\mu_s\,\cV(\cdot,\mu_s;\theta(s))\bigr) = 0, \quad s\in [0,1], \label{eq:measure}\\
&x_s|_{s=0} = x_0, \label{eq:init_token}\\
&\mu_s|_{s=0} = \frac{1}{n}\sum_{i=1}^n \delta_{z^{(i)}}. \label{eq:init_measure}
\end{empheq}
\end{subequations}
The final model output is given by the position of the distinguished particle at the terminal time,
\begin{align*}
    x_1 = x_1(x_0,\mu_0;\{\theta(s)\}_s)\in\RR^d.
\end{align*}
Here, $\theta(s)$ denotes the set of parameters at layer $s$. 
Equation~\eqref{eq:token} describes the distinguished token dynamics, 
\eqref{eq:measure} the contextual measure dynamics, 
\eqref{eq:init_token} the input distinguished token, 
and \eqref{eq:init_measure} the input contextual measure. For transformers, the self-attention vector field introduced in~\cite{sander2022sinkformers} is given by
\begin{equation}\label{eq:tf ic}
    \cV(x,\mu;\theta) = \frac{1}{Z(x)}\int_{\RR^d}
    \exp\bigl(\langle Q x, K y \rangle \bigr) V y \,\d\mu(y),
    \qquad
    Z(x) = \int_{\RR^d} \exp\bigl(\langle Q x, K z \rangle \bigr)\,\d\mu(z)\,.
\end{equation}
Here the parameter path $\theta(s)=(Q(s),K(s),V(s))$, $s\in[0,1]$, encodes the query, key, and value matrices across layers of Transformer blocks and is learned from data.

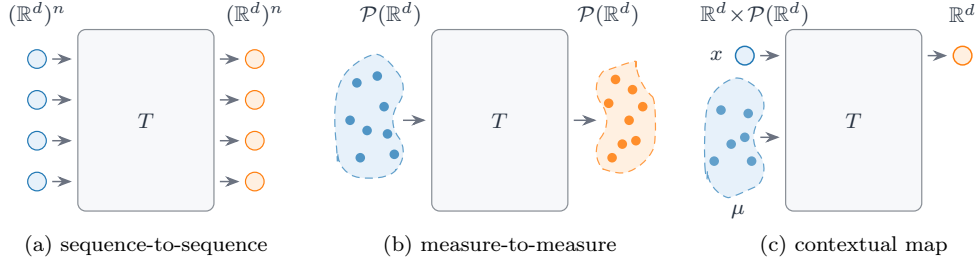
\begin{figure}[t]
\centering
\definecolor{neuInk}{HTML}{26313F}
\definecolor{neuArrow}{HTML}{6B7280}
\definecolor{neuBlock}{HTML}{F7F8FA}
\definecolor{cbBlue}{HTML}{1F77B4}   
\definecolor{cbOrange}{HTML}{FF7F0E} 
\definecolor{cbBlueFill}{HTML}{E7F1FA}
\definecolor{cbOrangeFill}{HTML}{FFF0E1}

\begin{tikzpicture}[
    x=0.9cm,
    y=0.9cm,
    >=Stealth,
    token/.style={circle, draw=cbBlue, line width=0.45pt, minimum size=7pt, inner sep=0pt, fill=cbBlueFill},
    token out/.style={circle, draw=cbOrange, line width=0.45pt, minimum size=7pt, inner sep=0pt, fill=cbOrangeFill},
    distinguished/.style={circle, draw=cbBlue, line width=0.55pt, minimum size=7pt, inner sep=0pt, fill=cbBlueFill},
    distinguished out/.style={circle, draw=cbOrange, line width=0.55pt, minimum size=7pt, inner sep=0pt, fill=cbOrangeFill},
    measure blob/.style={draw, rounded corners=8pt, inner sep=6pt, densely dashed, line width=0.45pt},
    measure blob in/.style={measure blob, draw=cbBlue!70, fill=cbBlueFill},
    measure blob out/.style={measure blob, draw=cbOrange!78, fill=cbOrangeFill},
    measure blob ctx/.style={measure blob, draw=cbBlue!70, fill=cbBlueFill},
    particle in/.style={fill=cbBlue!78},
    particle out/.style={fill=cbOrange!88},
    particle ctx/.style={fill=cbBlue!70},
    block/.style={draw=neuInk!55, line width=0.45pt, rounded corners=3pt, minimum width=18mm, minimum height=24mm, fill=neuBlock, text=neuInk},
    arr/.style={->, line width=0.55pt, draw=neuArrow, shorten <=2pt, shorten >=2pt},
    lbl/.style={font=\footnotesize\sffamily, text=neuInk},
  ]

  \begin{scope}[shift={(0,0)}]

    \foreach \i/\y in {1/0.9, 2/0.3, 3/-0.3, 4/-0.9} {
      \node[token] (in\i) at (-1.6, \y) {};
    }

    \node[block] (T1) at (0,0) {\footnotesize $T$};

    \foreach \i/\y in {1/0.9, 2/0.3, 3/-0.3, 4/-0.9} {
      \node[token out] (out\i) at (1.6, \y) {};
    }

    \foreach \i in {1,2,3,4} {
      \draw[arr] (in\i) -- (T1.west |- in\i);
      \draw[arr] (T1.east |- out\i) -- (out\i);
    }

    \node[lbl, above] at (-1.6, 1.25) {$(\mathbb{R}^d)^n$};
    \node[lbl, above] at ( 1.6, 1.25) {$(\mathbb{R}^d)^n$};

    \node[font=\footnotesize, anchor=north] at (0,-1.55) {(a) sequence-to-sequence};
  \end{scope}

  \begin{scope}[shift={(5.2,0)}]

    \coordinate (blobL) at (-1.8,0);
    \coordinate (pL1) at (-0.58,-0.8);
    \coordinate (pL2) at (-0.35,0.9);
    \coordinate (pL3) at (0.5,0.5);
    \coordinate (pL4) at (0.5,-0.5);
    \coordinate (cL1) at (-0.6,-0.3);
    \coordinate (cL2) at (-0.7,0.3);
    \coordinate (cL3) at (0.1,1.0);
    \coordinate (cL4) at (0.3,0.7);
    \coordinate (cL5) at (0.2,0.1);
    \coordinate (cL6) at (0.2,-0.1);
    \coordinate (cL7) at (0.3,-0.7);
    \coordinate (cL8) at (0.2,-0.9);
    \begin{scope}
      \draw[measure blob in] ($(blobL)+(pL1)$) 
        .. controls ($(blobL)+(cL1)$) and ($(blobL)+(cL2)$) ..
        ($(blobL)+(pL2)$) 
        .. controls ($(blobL)+(cL3)$) and ($(blobL)+(cL4)$) ..
        ($(blobL)+(pL3)$) 
        .. controls ($(blobL)+(cL5)$) and ($(blobL)+(cL6)$) .. %
        ($(blobL)+(pL4)$)                                         %
        .. controls ($(blobL)+(cL7)$) and ($(blobL)+(cL8)$) ..
        cycle;
      \foreach \px/\py in {-0.3/0.55, 0.1/0.2, -0.15/-0.15, 0.25/-0.5, -0.4/0.0, 0.0/0.65, 0.15/-0.2, -0.25/-0.55} {
        \fill[particle in] ($(blobL)+(\px,\py)$) circle (1.8pt);
      }


      
    \end{scope}

    \node[block] (T2) at (0,0) {\footnotesize $T$};

    \coordinate (blobR) at (1.8,0);
    \coordinate (pR1) at (-0.2,-0.55);
    \coordinate (pR2) at (-0.35,0.7);
    \coordinate (pR3) at (0.5,0.5);
    \coordinate (pR4) at (0.5,-0.5);
    \coordinate (cR1) at (-0.5,-0.3);
    \coordinate (cR2) at (-0.5,0.3);
    \coordinate (cR3) at (0.1,0.8);
    \coordinate (cR4) at (0.3,0.7);
    \coordinate (cR5) at (0.2,0.1);
    \coordinate (cR6) at (0.2,-0.1);
    \coordinate (cR7) at (0.3,-0.7);
    \coordinate (cR8) at (-0.2,-0.9);
    \begin{scope}
      \draw[measure blob out] ($(blobR)-(pR1)$) 
        .. controls ($(blobR)-(cR1)$) and ($(blobR)-(cR2)$) ..
        ($(blobR)-(pR2)$) 
        .. controls ($(blobR)-(cR3)$) and ($(blobR)-(cR4)$) ..
        ($(blobR)-(pR3)$) 
        .. controls ($(blobR)-(cR5)$) and ($(blobR)-(cR6)$) .. %
        ($(blobR)-(pR4)$)                                         %
        .. controls ($(blobR)-(cR7)$) and ($(blobR)-(cR8)$) ..
        cycle;
      \foreach \px/\py in {-0.1/0.6, 0.15/0.45, 0.3/0.2, 0.1/0.0, -0.2/0.25, 0.0/-0.35, -0.15/-0.55, 0.2/-0.3} {
        \fill[particle out] ($(blobR)+(\px,\py)$) circle (1.8pt);
      }


      
    \end{scope}

    \draw[arr] (blobL)+(0.3,0) -- (T2.west);
    \draw[arr] (T2.east) -- ($(blobR)-(0.3,0)$);

    \node[lbl, above] at (-1.6, 1.25) {$\mathcal{P}(\mathbb{R}^d)$};
    \node[lbl, above] at ( 1.6, 1.25) {$\mathcal{P}(\mathbb{R}^d)$};

    \node[font=\footnotesize, anchor=north] at (0,-1.55) {(b) measure-to-measure};
  \end{scope}

  \begin{scope}[shift={(10.4,0)}]

    \coordinate (blobC) at (-1.7,-0.25);
    \coordinate (pC1) at (-0.45,-0.8);
    \coordinate (pC2) at (-0.35,0.8);
    \coordinate (pC3) at (0.5,0.5);
    \coordinate (pC4) at (0.5,-0.5);
    \coordinate (cC1) at (-0.5,-0.3);
    \coordinate (cC2) at (-0.5,0.3);
    \coordinate (cC3) at (0.1,0.9);
    \coordinate (cC4) at (0.3,0.7);
    \coordinate (cC5) at (0.2,0.1);
    \coordinate (cC6) at (0.2,-0.1);
    \coordinate (cC7) at (0.3,-0.7);
    \coordinate (cC8) at (-0.1,-0.9);

    \node[block] (T3) at (0,0) {\footnotesize $T$};
    
    \begin{scope}
       \draw[measure blob ctx] ($(blobC)+(pC1)$) 
        .. controls ($(blobC)+(cC1)$) and ($(blobC)+(cC2)$) ..
        ($(blobC)+(pC2)$) 
        .. controls ($(blobC)+(cC3)$) and ($(blobC)+(cC4)$) ..
        ($(blobC)+(pC3)$) 
        .. controls ($(blobC)+(cC5)$) and ($(blobC)+(cC6)$) .. %
        ($(blobC)+(pC4)$)                                         %
        .. controls ($(blobC)+(cC7)$) and ($(blobC)+(cC8)$) ..
        cycle;
      \foreach \px/\py in {-0.25/0.4, 0.15/0.35, -0.1/-0.1, 0.2/-0.35, -0.35/-0.35, 0.1/0.0} {
        \fill[particle ctx] ($(blobC)+(\px,\py)$) circle (1.8pt);
      }
      
    \end{scope}

    \node[distinguished] (xin) at (-1.6, 0.95) {};
    \node[lbl, left=1pt] at (xin.west) {$x$};

    \node[distinguished out] (xout) at (1.6, 0.95) {};

    \draw[arr] (T3.east |- xin) -- (xout);
    \draw[arr] (xin) -- (T3.west |- xin);
    \draw[arr] (blobC)+(0.25,0) -- (T3.west |- blobC);

    \node[lbl, above] at (-1.45, 1.25) {$\mathbb{R}^d\!\times\!\mathcal{P}(\mathbb{R}^d)$};
    \node[lbl, above] at ( 1.6, 1.3) {$\mathbb{R}^d$};

    \node[lbl, below=0pt] at ($(blobC)+(0,-0.85)$) {$\mu$};

    \node[font=\footnotesize, anchor=north] at (0,-1.55) {(c) contextual map};
  \end{scope}

\end{tikzpicture}

\caption{Three formalizations of transformers.
}
\label{fig:cfm}
\end{figure}

The contextual measure $\mu_0=\frac{1}{n}\sum_i \delta_{z^{(i)}}$ can be regarded as an $n$-sample empirical estimator of an underlying population measure $\mu_0^\infty$. The contextual flow map~\eqref{eq:model intro} then takes this empirical measure as input and produces an output $x_1$ that approximates an idealized \emph{infinite-context} system, obtained by replacing $\mu_0$ with $\mu_0^\infty$ in~\eqref{eq:model intro}. The infinite-context system is a natural reference because an ideal transformer should be able to operate on arbitrarily long context: as the context grows, the contextual measure $\mu_0$ resolves the underlying population $\mu_0^\infty$ ever more faithfully, and a model that fully exploits its context should reflect this in its output. The infinite-context system is exactly the object that does so, and the finite-context model should be understood as approximating it. From this perspective, the context length $n$ plays the role of a statistical resource: it is the budget the model is given to estimate the population on which the ideal transformer would act.  We are thus led to the central question of the paper:
\begin{quote}
\emph{How accurately does a finite-context contextual flow map approximate its infinite-context limit, as a function of the context length~$n$, both at inference time (forward pass) and during training (backward pass)?}
\end{quote}
This question is of practical as well as theoretical interest: although recent advances have substantially extended the usable context window of large language models~\cite{team2024gemini}, long contexts remain constrained by computational costs, which in turn limit the emergence of new capabilities. Quantifying the dependence on $n$ therefore informs how aggressively context can be reduced without compromising fidelity to the infinite-context reference. The McKean--Vlasov structure of~\eqref{eq:model intro} makes it tractable through the classical machinery of propagation of chaos~\cite{kac1956foundations,sznitman1991topics,chaintrondiez2022propagationI,chaintrondiez2022propagationII}.

\medskip
\noindent{\bf Related work.}
The input-output behavior of transformers can be formalized in at least three ways: as a sequence-to-sequence map, a measure-to-measure map, and a contextual map; see Figure~\ref{fig:cfm}. The sequence-to-sequence viewpoint models a transformer $T:(\R^d)^n\to (\R^d)^n$ as a map from $n$ input tokens $(x_1,\ldots,x_n)$ to $n$ output tokens $(y_1,\ldots,y_n)$. This early and direct formalization was used to establish approximation results in~\cite{yun2019transformers, alberti2023sumformer}, albeit with architectures whose size grows exponentially in $n$. Although this dependence can be alleviated by allowing the embedding dimension to scale with $n$, the resulting growth in complexity remains restrictive in modern long-context applications. Early measure-theoretic foundations for attention were developed in~\cite{vuckovic2020mathematical}, which interpreted self-attention as a system of self-interacting particles and established regularity properties such as Lipschitz continuity under suitable assumptions. This line of work was substantially extended in the seminal paper~\cite{sander2022sinkformers}, where transformers were analyzed as a mean-field interacting particle system implementing a measure-to-measure map. This perspective has enabled the use of powerful tools from partial differential equations and Wasserstein gradient flows; see~\cite{geshkovski2023emergence, geshkovski2024dynamic,bruno2025emergence,bruno2025multiscale, chen2025quantitative, castin2025unified, rigollet2025mean, bruno2026scaling, chen2026critical, geshkovskiMeasuretomeasureInterpolationUsing2026, alvarezlopez2026perceptrons, alcalde2025frankwolfe}. From a statistical viewpoint, a closely related recent contribution is~\cite{kawata2026transformers}, which studies transformers as measure-theoretic associative memory and establishes generalization guarantees together with a matching minimax lower bound under spectral assumptions. Despite its usefulness, however, the measure-to-measure viewpoint does not explicitly capture the central role of the distinguished token. This limitation is addressed by the contextual-map formulation introduced in the~\cite{furuyaTransformersAreUniversal2024}, which established universal approximation within the class of contextual flow maps using architectures whose size is independent of the context length. The relevance of this contextual framework extends well beyond approximation theory. It provides a natural language for the rapidly growing literature on \emph{in-context learning} with transformers \cite{brown2020language, xie2021explanation, garg2022transformers,  bai2023transformers, wakayama2026incontext, samworth26, barnfield2026multilayer}. Moreover, it captures the structure of settings in which transformers are explicitly deployed as \emph{contextual velocity fields}, notably in diffusion models~\cite{peebles2023scalable} and measure-to-measure regression~\cite{LazM2M}.
The present work adopts the same contextual viewpoint, but focuses on statistical questions rather than expressivity alone.

The statistical behavior of transformers as the context length grows has received relatively little attention. 
Recent works~\cite{boursier2025softmax, bohbot2025token} study the convergence of a single attention head applied to $n$ i.i.d. tokens to its infinite-token limit. While both establish a parametric $n^{-1/2}$ rate of convergence, the latter draws a more subtle picture near the hardmax limit. Our contribution differs in two major ways. First, we study a deep stack of attention layers using contextual flow maps rather than a single attention layer. In a deep stack the context tokens interact through the shared velocity field~$\cV$, so they lose their initial independence after the first layer; bounding the deviation from the infinite-context limit therefore requires tracking the growth of correlations across layers. Second, we also study the generalization of training dynamics from finite to infinite context length.

Our analysis is accordingly closer in spirit to the theory of \emph{propagation of chaos} for McKean--Vlasov dynamics. We emphasize that in our setting the term refers to quantitative mean-field stability bounds rather than the convergence of finite marginals to product measures as in~\cite{lacker2023hierarchies}. Nevertheless, we adopt it as it has become standard in the machine learning literature~\cite{chizat2018global, mei2018mean, glasgow2025mean, glasgow2025propagation,glasgow2026quantitative, glasgow2026uniform,chizatHiddenWidthDeep2025,chaintron2026resnets}. While our work does not require the full apparatus of the theory, we borrow its key structural insight: controlling deviations through regularity of the velocity fields, in both the forward and the backward pass. Uniform-in-time propagation of chaos for the backward pass of two-layer feedforward networks---simple maps $f_\theta:\R^d\to\R^d$---was recently established in~\cite{guillinUniformintimeConcentrationTwolayer2026}. We extend these results to contextual flow maps, which (i) account for depth and (ii) depend on an evolving measure.

\medskip
\noindent{\bf Our contributions.} We initiate a quantitative statistical theory of contextual flow maps in the long-context regime, treating the context length~$n$ as the central statistical resource. Under mild regularity assumptions on the velocity field~$\cV$, our main results are the following.
\begin{itemize}
    \item[\textbf{(C1)}] \emph{Forward propagation of chaos.} We establish a quantitative deviation bound between the finite-context contextual flow map~\eqref{eq:model intro} and its infinite-context limit, holding uniformly along the depth of the network (Theorem~\ref{thm:forward POC main}). The analysis explicitly tracks the propagation of correlations introduced at each layer by the shared velocity field---a mechanism that is absent from the single-layer analyses of~\cite{boursier2025softmax, bohbot2025token}.
    \item[\textbf{(C2)}] \emph{Backward propagation of chaos.} We obtain a deviation bound between the Online Gradient Descent (OGD) training trajectory of a finite-context model and that of its infinite-context counterpart, holding over all iterations of the algorithm (Theorem~\ref{thm:OGD POC main}). This extends the recent two-layer feedforward result of~\cite{guillinUniformintimeConcentrationTwolayer2026} along two structural axes: \emph{(i)} depth, through a stack of attention layers with parameters shared across tokens, and \emph{(ii)} interaction with an evolving contextual measure that is itself transported across layers.
    \item[\textbf{(C3)}] \emph{Adjoint equation for contextual flows.} We derive an Eulerian adjoint formulation of the loss gradient (Theorem~\ref{thm:gradient adjoint} in Section~\ref{sec:adjoint}) that couples a covector adjoint for the distinguished token to a scalar adjoint test function for the contextual measure. This formulation underlies all of our stability analysis and extends naturally to the Riemannian setting relevant to post-LayerNorm transformers.
    \item[\textbf{(C4)}] \emph{Stability estimates.} Our results rest on a new set of stability estimates for the contextual flow and the loss gradient (Theorems~\ref{thm:stability flow main} and~\ref{thm:stability gradient main} in Section~\ref{sec:stability}), established uniformly in depth and uniformly along the parameter path. We further show in Section~\ref{sec:stab-trans} that the standard transformer architecture---both self-attention and MLP layers---satisfies these estimates with Lipschitz constants that are \emph{independent of the embedding dimension and parameter count}, suggesting a structural reason why transformers train stably across model scales.
\end{itemize}

\noindent {\bf Assumptions.} Our results rest on two sets of assumptions, stated as Assumptions~\ref{ass:regularity} and~\ref{ass:kernel form V}. The first is purely technical: routine regularity conditions (compactness, smoothness of the various objects entering the definition of CFMs) of the kind required by any analysis of this type. Some can likely be relaxed, but regularity of some form is unavoidable.

The second assumption plays a more substantive role. It is not needed to establish our slow rates $n^{-1/d}$; rather, it is the precise structural ingredient that upgrades these to the parametric rate $n^{-1/2}$. It posits that the velocity field $\cV$ driving the CFM~\eqref{eq:model intro} takes the form
\begin{equation*}
\cV(x,\mu;\theta) = \cF\left(\int_{\RR^d} F(x,y;\theta)\, \d \mu(y)\right) +G(x;\theta).
\end{equation*}
for sufficiently regular functions $\cF, F$ and $G$, so that $\mu$ is applied \emph{linearly}. Crucially, this is precisely the structure of self-attention~\eqref{eq:tf ic} with a multi-layer perceptron (MLP): the transformer-based CFMs that motivate this work satisfy Assumption~\ref{ass:kernel form V} \emph{by construction}, and therefore enjoy the fast parametric rate. The full technical statements are deferred to the appendix.

\medskip
\noindent {\bf Notation.}
We denote the Euclidean inner product and norm by $ \langle \cdot , \cdot \rangle$ and $|\cdot|$ respectively. The operator norm of a matrix $A$ is denoted by $\|A\|:=\sup_{|h|=1}|Ah|$. For a map $\theta:[0,1] \to \R^p$, we use the following norms $\|\theta\|_{L^1}:=\int_0 ^1 |\theta(s)|d s$ and $\|\theta\|_{L^\infty}:=\sup_{s \in [0,1]} |\theta(s)|$. 
The 1-Wasserstein distance~\cite[see, e.g.,][Def. 1.2]{CheNilRig25} between two probability measures $\mu,\nu\in\cP(\RR^d)$ is denoted by $\bW_1(\mu,\nu)$. 
For a velocity field $\cV(x,\mu;\theta):\RR^d\times \cP(\RR^d)\times\RR^p\to\RR^d$, $D_x \cV(x,\mu;\theta)$ and  $D_{\theta} \cV(x,\mu;\theta)$ denote the $d\times d$ (resp. $d \times p$) Jacobian matrix with respect to $x$ (resp. $\theta$) and $\frac{\delta \cV}{\delta \mu}[x,\mu;\theta](z)$ denotes the Fr{\'e}chet derivative with respect to $\mu$ and is a vector field in $\RR^d$ indexed by $z$. 
    The Wasserstein Jacobian (a.k.a the Lions derivative) is defined as the (spatial) Jacobian of the Fr{\'e}chet derivative:
    $\gradW_\mu \cV[x, \mu; \theta](z) = D_z \left( \frac{\delta \cV}{\delta \mu}[x, \mu; \theta](z) \right)$.

\section{Main results}\label{sec:main}

We work throughout with the contextual flow map~\eqref{eq:model intro}, which sends an input pair
$(x_0,\mu_0)\in\RR^d\times\cP(\RR^d)$ together with a parameter path
$\theta:[0,1]\to\RR^p$ to the output token $x_1=x_1(x_0,\mu_0;\theta)\in\RR^d$.
All results below are stated for the inputs and the velocity field $\cV$: (i) under the standard regularity hypotheses collected in Assumption~\ref{ass:regularity}; or (ii) under the special form of $\cV$ specified in \Cref{ass:kernel form V}. We verify in Appendix~\ref{sec:verify transformer} that both assumptions are met by the self-attention vector field~\eqref{eq:tf ic} and standard MLP layers with dimension-free constants. 
We let $\bA$ denote the bundle of constants appearing in the assumptions, and write $C=C(\bA)$ for any positive constant depending only on these parameters (in particular, independent of $p$, $d$, and $n$).

\paragraph{Loss functional.}
Given a target $y_0\in\RR^d$ and a point loss
$\ell:\RR^d\times\RR^d\to\RR$ (e.g.\ squared error or cross-entropy), the
prediction loss of a parameter path $\theta$ on a sample $(x_0,\mu_0,y_0)$ is
\begin{equation}\label{eq:loss}
    \cL(\theta;\mu_0,x_0,y_0)
    \;\coloneqq\;
    \ell\!\left(x_1(x_0,\mu_0;\theta),\,y_0\right).
\end{equation}
Because the dependence of $x_1$ on the entire path
$\theta(\cdot)$ is mediated implicitly through the coupled token--measure
system~\eqref{eq:model intro}, computing the Fr\'echet derivative
$\delta\cL/\delta\theta$ is non-trivial; an explicit Eulerian-adjoint formula
is derived in Appendix~\ref{sec:adjoint system} and underlies the regularity
analysis below.

To obtain better stability properties, we consider the loss functional with a ridge penalty:
\begin{align}\label{eq:main loss functional}
    \cL(\theta;\mu_0,x_0,y_0) + \frac{\lambda}{2}\int_0^1|\theta(s)|^2\,\d s.
\end{align}
In Theorem~\ref{thm:OGD POC main} below, we establish two results: one for $\lambda$ small (or zero) and one for $\lambda$ large. The latter yields stability across all iterations of online gradient descent.

\paragraph{Online gradient descent.}
We train the model by online gradient descent (OGD)\footnote{OGD reduces to Stochastic Gradient Descent if we further assume that the samples are i.i.d.} on a stream of 
samples. At each iteration $k\in\NN$,  the parameter path is updated using a triple of observations
$(x_0^k,\mu_0^k,y_0^k)$ according to
\begin{equation}\label{eq:OGD}
    \theta_{k+1}(s)
    \;=\;
    (1-\eta\lambda)\,\theta_k(s)
    \;-\;
    \eta\,\frac{\delta\cL(\theta_k;\mu_0^k,x_0^k,y_0^k)}{\delta\theta}(s),
    \qquad s\in[0,1],
\end{equation}
where $\eta>0$ is a fixed step size and $\lambda\in[0,\eta^{-1})$.

OGD is a faithful idealization of how large transformers are pretrained: a (close to) single pass over an enormous corpus, in which each step sees a new batch that need not be independent of other observations.

\subsection{Propagation of chaos for the forward pass}\label{sec:forward POC}

We first analyze the inference-time behavior. Fix a parameter path $\theta$,
a distinguished token $x_0\in\RR^d$, and a contextual measure
$\mu_0\in\cP(\RR^d)$, and let $z^{(1)},\ldots,z^{(n)}\stackrel{\mathrm{iid}}{\sim}\mu_0$.
We compare the contextual flow driven by the population context $\mu_0$ to the
one driven by the empirical context
\begin{equation}\label{eq:empirical context}
    \widehat\mu_0
    \;\coloneqq\;
    \frac{1}{n}\sum_{i=1}^n \delta_{z^{(i)}},
\end{equation}
denoting the corresponding solutions of~\eqref{eq:model intro} by
$(x_s,\mu_s)$ and $(\widehat x_s,\widehat\mu_s)$.

The central point is that the bound below is \emph{uniform in the
depth variable~$s \in [0,1]$}, including at output time $s=1$. Such a result is non-trivial since after the first layer the $n$
context tokens are no longer independent: they are coupled through the
shared velocity field $\cV$. Controlling the
build-up of correlations across layers is precisely what propagation-of-chaos arguments are designed for in the McKean--Vlasov
literature~\cite{sznitman1991topics, chaintrondiez2022propagationI},
and it is what we adapt here to the contextual setting. The proof of the next theorem uses a standard Gr\"onwall bound. It does not suffer from a classical exponential deterioration since we assume $s \le 1$.

\begin{theorem}[Forward propagation of chaos]\label{thm:forward POC main}
Adopt Assumption~\ref{ass:regularity} and suppose that $d\geq 3$. There exists
a constant $C=C(\bA)$ such that, with probability at least
$1-4\exp\!\left(-n^{\,1-2/d}\right)$,
\begin{equation}\label{eq:forward POC bound}
    \sup_{s\in[0,1]} \bW_1\!\left(\mu_s,\widehat\mu_s\right)
    \;\leq\; C\,n^{-1/d}
    \qquad\text{and}\qquad
    \sup_{s\in[0,1]}\, \bigl|x_s-\widehat x_s\bigr|
    \;\leq\; C\,n^{-1/d}.
\end{equation}
If $\cV$ additionally satisfies \Cref{ass:kernel form V}, then for every $\delta \in (0,1/2)$,
\begin{equation}\label{eq:intro forward POC sharp bound}
    \sup_{s\in[0,1]}\, \bigl|x_s-\widehat x_s\bigr|
    \;\leq\; C\,\sqrt{\frac{\log(\delta^{-1})}{n}},
\end{equation}
with probability at least $1-\delta$.
\end{theorem}

A few comments are in order. First regarding the slow rate~\eqref{eq:forward POC bound}: without further structural assumptions on the CFM such as Assumption~\ref{ass:kernel form V}, the rate $n^{-1/d}$ is the best one can hope for, as it is already a bottleneck at the level of the input contextual measure by the Wasserstein law of large numbers~\cite[Ch.~2]{CheNilRig25}. 
The contextual flow therefore does \emph{not} amplify the empirical noise; the finite-context output is statistically as good as the input it consumes. 
When $\mu_0$ is supported on a submanifold of intrinsic dimension $d'<d$, the same bound holds with $n^{-1/d}$ replaced by $n^{-1/d'}$~\cite{WeeBac19}; this is the practically relevant regime in transformer pretraining, where token embeddings are believed to concentrate on a low-dimensional manifold despite the ambient dimension $d$ being in the hundreds or thousands. 
The condition $d\geq 3$ is purely cosmetic: the more general Theorem~\ref{thm: concentration of empirical contextual flow} in Appendix~\ref{sec:stability of contextual flow} holds for every $d\geq 1$, with the rate $n^{-1/d}$ adjusted according to the Wasserstein law of large numbers. 
More generally, these bounds instantiate a pattern recently articulated by Raginsky and Recht~\cite{raginsky2026separating}: the deterministic core of the result is a Wasserstein-stability statement for the flow, into which the data assumption is plugged ex post via a Wasserstein law of large numbers, separating the geometry of the dynamics from the probabilistic assumptions on the data.

The fast/parametric rate~\eqref{eq:intro forward POC sharp bound} is quite striking: despite the presence of multiple layers that compromise independence, the parametric rate persists thanks to propagation of chaos across layers for structured CFMs such as transformers. Details are provided in Appendix~\ref{sec:parametric}.
\subsection{Propagation of chaos along online gradient descent}\label{sec:OGD POC}

We now turn to the corresponding question on the backward pass, which is
significantly more delicate. Let $\{(x_0^k,\mu_0^k,y_0^k)\}_{k\in\NN}$ be the population (infinite context) data stream and, at each step $k$, form the empirical contextual measure
\begin{equation}\label{eq:empirical context training}
    \widehat\mu_0^{\,k}
    \;\coloneqq\;
    \frac{1}{n}\sum_{i=1}^n \delta_{z^{(k,i)}}.
\end{equation}
where $z^{(k,i)}\sim \mu_0^k$ are independent for all $k\in\NN$ and $i \in \{1,2,\dots,n\}$.
Starting both trajectories from the same path $\theta_0$, we compare the two OGD
trajectories produced by~\eqref{eq:OGD}: the \emph{population OGD}
$\{\theta_k\}_{k\in\NN}$ driven by $(x_0^k,\mu_0^k,y_0^k)$, and the
\emph{empirical OGD} $\{\widehat\theta_k\}_{k\in\NN}$ driven by
$(x_0^k,\widehat\mu_0^{\,k},y_0^k)$ with $\widehat \theta_0 = \theta_0$. The statistical quantity of interest is
the deviation $\|\theta_k-\widehat\theta_k\|_{L^\infty}$.

The behavior of this deviation depends on the regularization parameter $\lambda$. 
Without enough regularization, parameter trajectories may drift apart after a large number of iterations and the deviation can in principle accumulate; 
strong enough regularization, by contrast, contracts the dynamics at every step and yields a bound at any prescribed iteration.

\begin{theorem}[Backward propagation of chaos]\label{thm:OGD POC main}
Adopt Assumption~\ref{ass:regularity} and suppose that $d\geq 3$. 
There exist constants $C=C(\bA)$, $\eta_c=\eta_c(\bA)$, and $\lambda_c=\lambda_c(\bA)<\eta_c^{-1}$ such that the following hold.
\begin{enumerate}
    \item[\emph{(i)}] \emph{Small or no ridge penalty.}
    For $\eta>0$ and  $\lambda\in[0,\eta^{-1})$, we let $k_c = (\log(1+C\eta))^{-1}$. Then, for any $k \leq k_c$
    \begin{equation}\label{eq:finite OGD POC}
         \bigl\|\theta_k-\widehat\theta_k\bigr\|_{L^\infty}
        \;\leq\; C\,n^{-1/d},
    \end{equation}
    with probability at least $1-2\exp\!\left(-\eta^{-1}n^{\,1-2/d}\right)$. 
    
    \noindent If $\cV$ additionally satisfies \Cref{ass:kernel form V}, then for every $\delta \in (0,1/2)$, $k \leq k_c$,
    \begin{equation}\label{eq:intro finite OGD POC sharp}
        \bigl\|\theta_k-\widehat\theta_k\bigr\|_{L^\infty}
        \;\leq\; C\, \left(1+\sqrt{\eta\log(\delta^{-1})} \right) \sqrt{\frac{d}{n}},
    \end{equation}
    with probability at least $1-\delta$.

    \item[\emph{(ii)}] \emph{Large ridge penalty.}
    If $\eta<\eta_c$ and $\lambda\in(\lambda_c,\eta^{-1})$, then for every $k\in\NN$
    \begin{equation}\label{eq:uniform OGD POC}
        \bigl\|\theta_k-\widehat\theta_k\bigr\|_{L^\infty}
        \;\leq\; C\,n^{-1/d},
    \end{equation}
    with probability at least $1-2\exp\!\left(-\eta^{-1}\,n^{\,1-2/d}\right)$.
    
    \noindent If $\cV$ additionally satisfies \Cref{ass:kernel form V}, then for every $\delta \in (0,1/2)$, $k\in\NN$,
    \begin{equation}\label{eq:intro uniform OGD POC sharp}
        \bigl\|\theta_k-\widehat\theta_k\bigr\|_{L^\infty}
        \;\leq\; C\, (1+\sqrt{\eta\log(\delta^{-1})}) \sqrt{\frac{d}{n}},
    \end{equation}
    with probability at least $1-\delta$.
\end{enumerate}
\end{theorem}
The proofs are provided in Appendix~\ref{sec:OGD stability} and Appendix~\ref{sec:parametric}. Also, three features of this result are worth emphasizing.


\noindent\emph{Rate.} Both bounds~\eqref{eq:finite OGD POC},\eqref{eq:uniform OGD POC} match the inference-time rate $n^{-1/d}$ of Theorem~\ref{thm:forward POC main} (sharpening to $n^{-1/d'}$ on a $d'$-dimensional submanifold) and $n^{1/2}$ under Assumption~\ref{ass:kernel form V}: training incurs no statistical loss beyond, even though the OGD iterates depend on the full history of empirical samples and gradients.


\noindent\emph{Role of regularization.} The dichotomy between (i) and (ii) is genuine. Statement~(i) is unconditional but limited to a horizon $k\lesssim 1/\eta$, beyond which Gr\"onwall bounds deteriorate exponentially; statement~(ii) trades this for sufficient ridge regularization $\lambda>\lambda_c$, ensuring that the per-step contraction beats the noise injected by replacing $\mu_0^k$ with $\widehat\mu_0^{\,k}$.


\noindent\emph{Comparison with prior work.} The result extends the uniform-in-time backward propagation of chaos of~\cite{guillinUniformintimeConcentrationTwolayer2026} along two structural axes: \emph{depth}, through a stack of attention layers with parameters shared across tokens, and \emph{measure interaction}, through a contextual measure transported by the dynamics. The technical crux is the Lipschitz dependence of $\delta\cL/\delta\theta$ on the contextual measure (Section~\ref{sec:stability}), which transfers forward concentration into training stability.


\section{The adjoint equation}\label{sec:adjoint}
 
The proof of Theorem~\ref{thm:OGD POC main} hinges on a precise analytical handle on the gradient of the loss with respect to the parameter path~$\theta$. Because the dependence of $x_1$ on $\theta(\cdot)$ is mediated implicitly through the coupled token--measure system~\eqref{eq:model intro}, this gradient does not admit a closed-form chain-rule expression. It is instead given by an \emph{Eulerian adjoint} construction in which two Lagrange multipliers---one for the token dynamics~\eqref{eq:token} and one for the contextual measure dynamics~\eqref{eq:measure}---propagate the error signal backward through depth. We state the result here and defer the derivation to Appendix~\ref{sec:adjoint system}.
 
Recall from~\eqref{eq:loss} that, given a target $y_0\in\RR^d$ and a smooth point loss $\ell:\RR^d\times\RR^d\to\RR$, the prediction loss of a parameter path $\theta$ on the sample $(x_0,\mu_0,y_0)$ is
$\cL(\theta;\mu_0,x_0,y_0)\;\coloneqq\;\ell\bigl(x_1(x_0,\mu_0;\theta),y_0\bigr)$
where $x_1$ is obtained by integrating~\eqref{eq:model intro} from depth $s=0$ to $s=1$.
 
\begin{theorem}[Gradient via adjoint variables]\label{thm:gradient adjoint}
Adopt Assumption~\ref{ass:regularity}. Let $(x_s,\mu_s)$ solve the contextual flow~\eqref{eq:model intro} with initial condition $(x_0,\mu_0) \in \RR^d \times \cP(\RR^d)$. Then there exist a covector path $p_\cdot:[0,1]\to\RR^d$ and a scalar field $\phi_\cdot:[0,1]\to C^{1,1}(\RR^d)$ such that the Fr\'echet derivative of the loss with respect to the parameter path satisfies
\begin{equation}\label{eq:gradient}
    \frac{\delta\cL(\theta;\mu_0,x_0,y_0)}{\delta\theta}(s)
    \;=\;
    D_\theta\cV(x_s,\mu_s;\theta(s))^{\top}\,p_s
    \;+\;
    \int_{\RR^d}D_\theta\cV(z,\mu_s;\theta(s))^{\top}\,\nabla\phi_s(z)\,\d\mu_s(z),
\end{equation}
where the \emph{token adjoint} $p_s$ and the \emph{measure adjoint} $\phi_s$ solve the backward system
\begin{equation}\label{eq:adjoint token}
    \begin{cases}
    \dot p_s\;=\;-D_x\cV(x_s,\mu_s;\theta(s))^{\top}\,p_s,\\[2pt]
    p_1\;=\;\nabla_x\ell(x_1,y_0),
    \end{cases}
\end{equation}
\begin{equation}\label{eq:adjoint measure}
    \begin{cases}
    \partial_s\phi_s(z)+\nabla\phi_s(z)\cdot\cV(z,\mu_s;\theta(s))\\[2pt]
    \quad=\;-\,p_s\cdot\dfrac{\delta\cV}{\delta\mu}[x_s,\mu_s;\theta(s)](z)\;-\;\displaystyle\int_{\RR^d}\nabla\phi_s(\xi)\cdot\dfrac{\delta\cV}{\delta\mu}[\xi,\mu_s;\theta(s)](z)\,\d\mu_s(\xi),\\[6pt]
    \phi_1(z)\equiv 0.
    \end{cases}
\end{equation}
Moreover, the system~\eqref{eq:adjoint token}--\eqref{eq:adjoint measure} is well-posed, and there exists a constant $C=C(\bA)$ such that, for every $s\in[0,1]$ and $z,\widetilde z\in\RR^d$,
\begin{equation}\label{eq:adjoint regularity}
    |p_s|\leq C,\qquad |\phi_s(z)|\leq C(|z|+1),\qquad |\nabla\phi_s(z)|\leq C,\qquad |\nabla\phi_s(z)-\nabla\phi_s(\widetilde z)|\leq C|z-\widetilde z|.
\end{equation}
\end{theorem}
 
The structure of~\eqref{eq:adjoint token}--\eqref{eq:adjoint measure} deserves comment. The covector $p_s$ runs backward in depth and tracks the sensitivity of the terminal loss to the position of the distinguished token at depth $s$; it is a direct generalization of the classical adjoint variable in the neural-ODE literature~\cite{chen2018neural}. The scalar field $\phi_s$ plays the analogous role for the contextual measure: it tracks the sensitivity of the loss to perturbations of $\mu_s$ at depth $s$, and its evolution couples back to the token adjoint through the source term involving $\delta\cV/\delta\mu$. The $C^{1,1}$-regularity stated in~\eqref{eq:adjoint regularity}---and in particular the Lipschitz continuity of $\nabla\phi_s$---is what ultimately enables Wasserstein-stability bounds on the gradient~\eqref{eq:gradient}: a less regular adjoint test function would not interact stably with empirical contextual measures.
 
\paragraph{Derivation in one sentence.} The adjoint system~\eqref{eq:adjoint token}--\eqref{eq:adjoint measure} arises by introducing a Lagrangian for the constrained optimization problem of minimizing $\cL$ subject to~\eqref{eq:model intro}, with $p_s$ enforcing the token equation and $\phi_s$ enforcing the continuity equation, and then choosing the multipliers so that the variations of the Lagrangian with respect to $x_s$ and $\mu_s$ vanish. The full computation, including the integration-by-parts manipulations that produce the cross terms in~\eqref{eq:adjoint measure} and the well-posedness analysis of~\eqref{eq:adjoint measure} via a contraction argument in a weighted $C^1$ space, is given in Appendix~\ref{sec:adjoint system}.

\paragraph{Relation to control theory.} Adjoint systems featuring a Wasserstein-derivative term are well-known in the Pontryagin theory of McKean--Vlasov control~\cite{carmona2015forward, carmona2018probabilistic, bensoussan2013mean} and in the optimal control of continuity equations~\cite{bonnet2021necessary}. The structure of~\eqref{eq:adjoint token}--\eqref{eq:adjoint measure} is closest in spirit to the major--minor mean-field formalism~\cite{huang2010large, carmona2016probabilistic}, in which a distinguished agent interacts with a continuum population. Our contribution is not only the specific adjoint equation for contextual flow maps but also the stability estimates of Section~\ref{sec:stability} which are absent from the control literature.

\paragraph{Riemannian extension.} The Eulerian Lagrange-multiplier argument is geometrically transparent and extends without modification to settings where the contextual flow is constrained to a Riemannian submanifold $\bM\subseteq\RR^d$---for example, the unit sphere $\SS^{d-1}$ relevant to post-LayerNorm transformers as in~\cite{geshkovski2025mathematical}. In that case $p_s$ becomes a tangent vector in $T_{x_s}\bM$, the spatial Jacobian $D_x\cV$ is replaced by the covariant derivative $\nabla^{\bM}\cV$, and~\eqref{eq:adjoint token} is modified accordingly.
 
\section{Stability estimates}\label{sec:stability}
 
The forward and backward propagation-of-chaos bounds of Section~\ref{sec:main} ultimately reduce to a  number of \emph{stability estimates}: quantitative statements asserting that the contextual flow~\eqref{eq:model intro}, the adjoint system~\eqref{eq:adjoint token}--\eqref{eq:adjoint measure}, and the loss gradient~\eqref{eq:gradient} all depend Lipschitz-continuously on their inputs, with constants that are uniform in depth. We collect the two estimates that drive the analysis in this section, and outline how they imply the main theorems. Full proofs are deferred to Appendices~\ref{sec:stability of contextual flow} and~\ref{sec:adjoint system}.
 
\subsection{Stability of the contextual flow}
 
Let $(x_s,\mu_s)$ and $(\widetilde x_s,\widetilde\mu_s)$ denote the solutions of~\eqref{eq:model intro} associated with two parameter paths $\theta,\widetilde\theta$ and two initial conditions $(x_0,\mu_0)$ and $(\widetilde x_0,\widetilde\mu_0)$ in $\RR^d \times \cP(\RR^d)$, respectively.
 
\begin{theorem}[Stability of the contextual flow]\label{thm:stability flow main}
Adopt Assumption~\ref{ass:regularity}. There exists $C=C(\bA)$ such that, for every $s\in[0,1]$,
\begin{equation}\label{eq:stab mu main}
    \bW_1(\mu_s,\widetilde\mu_s)\;\leq\;C\bigl(\bW_1(\mu_0,\widetilde\mu_0)\,+\,\|\theta-\widetilde\theta\|_{L^1}\bigr),
\end{equation}
\begin{equation}\label{eq:stab x main}
    |x_s-\widetilde x_s|\;\leq\;C\bigl(\bW_1(\mu_0,\widetilde\mu_0)\,+\,|x_0-\widetilde x_0|\,+\,\|\theta-\widetilde\theta\|_{L^1}\bigr).
\end{equation}
\end{theorem}
 
The proof, given in Appendix~\ref{sec:stability of contextual flow}, follows the classical Gr\"onwall pattern, lifting the difference between flows to a coupling of trajectories along their common characteristics. The key point is that the time horizon is fixed at $s=1$, which prevents the exponential blow-up typical of long-time Gr\"onwall arguments and yields dimension-free constants.
 
Theorem~\ref{thm:stability flow main} immediately yields the forward propagation of chaos bound: instantiating it with the empirical contextual measure $\widehat\mu_0=\frac1n\sum_{i=1}^n\delta_{z^{(i)}}$ in place of $\widetilde\mu_0$, with the same parameter path and the same initial token, gives
\[
\sup_{s\in[0,1]}\bW_1(\mu_s,\widehat\mu_s)\;+\;\sup_{s\in[0,1]}|x_s-\widehat x_s|\;\leq\;C\,\bW_1(\mu_0,\widehat\mu_0),
\]
and Theorem~\ref{thm:forward POC main} then follows by combining this deterministic bound with the Wasserstein law of large numbers and a McDiarmid-type concentration inequality (Appendix~\ref{sec:stability of contextual flow}).
 
\subsection{Stability of the loss gradient}
 
The backward analysis requires substantially more: it requires that the loss gradient~\eqref{eq:gradient} inherit the stability of the forward flow. This is the most technically demanding ingredient of the paper, because the gradient depends on the entire trajectory $(x_s,\mu_s,p_s,\phi_s)$ of the coupled forward--adjoint system.
 
\begin{theorem}[Stability of the loss gradient]\label{thm:stability gradient main}
Adopt Assumption~\ref{ass:regularity}. There exists $C=C(\bA)$ such that, for every $s\in[0,1]$,
\begin{equation}\label{eq:stab grad main}
    \begin{aligned}
    &\biggl|\frac{\delta\cL(\theta;\mu_0,x_0,y_0)}{\delta\theta}(s)\;-\;\frac{\delta\cL(\widetilde\theta;\widetilde\mu_0,\widetilde x_0,\widetilde y_0)}{\delta\theta}(s)\biggr|\\
    &\qquad\leq\;C\Bigl(\bW_1(\mu_0,\widetilde\mu_0)+|x_0-\widetilde x_0|+|y_0-\widetilde y_0|+\|\theta-\widetilde\theta\|_{L^1}+|\theta(s)-\widetilde\theta(s)|\Bigr).
    \end{aligned}
\end{equation}
\end{theorem}
 
The bound is uniform in $s\in[0,1]$ and \emph{linear} in each of the input perturbations. The proof proceeds in three steps, all carried out in Appendix~\ref{sec:adjoint system}. \emph{(i)} Theorem~\ref{thm:stability flow main} provides the stability of the forward trajectory $(x_s,\mu_s)$. \emph{(ii)} The token adjoint $p_s$ obeys a linear ODE whose coefficients depend on $(x_s,\mu_s,\theta(s))$, so a Gr\"onwall argument transfers the stability of the forward flow into a stability bound for $p_s$. \emph{(iii)} The measure adjoint $\phi_s$ is the most delicate: it solves a transport equation with a non-local source term, and its regularity is established by a fixed-point argument in a weighted $C^1$ space (the same Banach space used to prove well-posedness in Theorem~\ref{thm:gradient adjoint}); the contraction constant degrades by a controlled amount under perturbation of the data, yielding a Lipschitz bound on $(\phi_s,\nabla\phi_s)$.
 
\smallskip
 
The Lipschitz dependence in~\eqref{eq:stab grad main} is exactly what is needed to transfer the forward concentration of Theorem~\ref{thm:forward POC main} into the stability of the OGD trajectory. Indeed, applying Theorem~\ref{thm:stability gradient main} to the population update $\theta_{k+1}=(1-\eta\lambda)\theta_k-\eta\,\delta\cL(\theta_k;\mu_0^k,x_0^k,y_0^k)/\delta\theta$ and its empirical counterpart driven by $\widehat\mu_0^k$ produces the recursion
\[
\|\theta_{k+1}-\widehat\theta_{k+1}\|_{L^\infty}\;\leq\;\bigl(1-\eta(\lambda-C)\bigr)\,\|\theta_k-\widehat\theta_k\|_{L^\infty}\;+\;\eta\,C\,\bW_1(\mu_0^k,\widehat\mu_0^k).
\]
In the regime $\lambda>\lambda_c$, the per-step contraction $1-\eta(\lambda-C)$ is strictly less than $1$, and iterating the recursion yields the bound~\eqref{eq:uniform OGD POC} in Theorem~\ref{thm:OGD POC main}; in the regime with small or no regularization, iterating only over $k\lesssim 1/\eta$ steps produces the finite-horizon bound~\eqref{eq:finite OGD POC}.
 
\subsection{Stability estimates for transformers}
 \label{sec:stab-trans}

 All the estimates above are stated under the abstract regularity assumptions collected in \Cref{ass:regularity} on the velocity field $\cV$. In Appendix~\ref{sec:verify transformer}, we verify that the concrete self-attention vector field~\eqref{eq:tf ic}, together with standard MLP layers, satisfies these assumptions. Crucially, all resulting Lipschitz constants (and higher-order regularity constants) are \emph{independent of the embedding dimension $d$ and the parameter dimension $p$}---i.e., they are truly dimension-free. Our analysis builds upon prior work on the smoothness of attention. For example, \cite{castin2023smooth, furuya2026approximation} established Lipschitz continuity of the attention mechanism with respect to the input tokens (corresponding to $D_x\cV$ and $\gradW_{\mu}\cV$ in our setting). In addition, we establish full Lipschitz continuity (together with the necessary higher-order regularity) with respect to \emph{both} the tokens \emph{and} the trainable parameter matrices $Q,K,V$ themselves. Combined with the dimension-free nature of our constants, all the rates stated in \Cref{thm:forward POC main} and \Cref{thm:OGD POC main} hold for realistic Transformer architectures with absolute constants; the only unavoidable dependence is on the input-support radius $R$ and the parameter-path bound $M$.

 \section{Assumptions}\label{sec:assumption}
 
We collect here the standard regularity hypotheses on the velocity field $\cV$, the loss function $\ell$, and the input data that are referenced throughout the paper. Different results require different subsets of these hypotheses; we make this explicit in each statement by citing the bundle $\bA$ of constants involved.
 
\begingroup
\renewcommand{\theassumption}{A}
\begin{assumption}[Regularity]\label{ass:regularity}
We collect the following regularity conditions on the input data, the loss $\ell$, and the velocity field $\cV$:
\begin{enumerate}[label=\textup{(A\arabic*)},ref=\textup{A\arabic*}]
    \item\label{ass:compact} \emph{(Data and initialization).} The contextual measure $\mu_0$ has compact support, and $x_0$, $y_0$ are uniformly bounded: there exist $R,M>0$ such that $\supp\mu_0\subseteq B_R(0)$, $|x_0|\leq R$, and $|y_0|\leq R$. The initial parameter is uniformly bounded and Lipschitz: $|\theta(s)|\leq M$ and $|\theta(s)-\theta(t)|\leq M|s-t|$ for all $s,t\in[0,1]$.
 
    \item\label{ass:l} \emph{(Smoothness of the loss).} There exist $M_\ell,L_\ell>0$ such that, for all $x,\widetilde x,y,\widetilde y\in\RR^d$,
    \begin{align*}
        |\nabla_x\ell(x,y)|\leq M_\ell(|x|+|y|+1),
        \qquad
        |\nabla_x\ell(x,y)-\nabla_x\ell(\widetilde x,\widetilde y)|\leq L_\ell\bigl(|x-\widetilde x|+|y-\widetilde y|\bigr).
    \end{align*}
 
    \item\label{ass:V} \emph{(Uniform Lipschitz continuity of $\cV$).} There exists $L_0>1$ such that, for all $x,\widetilde x,z,\widetilde z\in\RR^d$, $\mu,\widetilde\mu\in\cP(\RR^d)$, and $\theta,\widetilde\theta\in\RR^p$,
    \begin{align*}
        |\cV(x,\mu;\theta)-\cV(\widetilde x,\widetilde\mu;\widetilde\theta)|
        &\leq L_0\bigl(|x-\widetilde x|+\bW_1(\mu,\widetilde\mu)+|\theta-\widetilde\theta|\bigr),\\
        \Bigl|\tfrac{\delta\cV}{\delta\mu}[x,\mu;\theta](z)-\tfrac{\delta\cV}{\delta\mu}[\widetilde x,\widetilde\mu;\widetilde\theta](\widetilde z)\Bigr|
        &\leq L_0\bigl(|x-\widetilde x|+\bW_1(\mu,\widetilde\mu)+|\theta-\widetilde\theta|+|z-\widetilde z|\bigr).
    \end{align*}
    Furthermore, there exists $M_0>1$ such that
    \begin{align*}
        |\cV(x,\mu;\theta)|\leq M_0(|x|+|\theta|+1),
        \qquad
        \Bigl|\tfrac{\delta\cV}{\delta\mu}[x,\mu;\theta](z)\Bigr|\leq M_0(|x|+|\theta|+|z|+1).
    \end{align*}
 
    \item\label{ass:derivative} \emph{(Smoothness of the derivatives).} The Jacobians $D_x\cV$ and $D_\theta\cV$, and the Wasserstein Jacobian $\gradW_\mu\cV$, are uniformly bounded and Lipschitz continuous with respect to $(x,\mu,\theta)$: there exist $L_1,L_2,L_3>0$ such that, for all $x,\widetilde x,z,\widetilde z\in\RR^d$, $\mu,\widetilde\mu\in\cP(\RR^d)$, and $\theta,\widetilde\theta\in\RR^p$,
    \begin{align*}
        \|D_x\cV(x,\mu;\theta)-D_x\cV(\widetilde x,\widetilde\mu;\widetilde\theta)\|
        &\leq L_1\bigl(|x-\widetilde x|+\bW_1(\mu,\widetilde\mu)+|\theta-\widetilde\theta|\bigr),\\
        \|\gradW_\mu\cV(x,\mu;\theta)(z)-\gradW_\mu\cV(\widetilde x,\widetilde\mu;\widetilde\theta)(\widetilde z)\|
        &\leq L_2\bigl(|x-\widetilde x|+\bW_1(\mu,\widetilde\mu)+|\theta-\widetilde\theta|+|z-\widetilde z|\bigr),\\
        \|D_\theta\cV(x,\mu;\theta)-D_\theta\cV(\widetilde x,\widetilde\mu;\widetilde\theta)\|
        &\leq L_3\bigl(|x-\widetilde x|+\bW_1(\mu,\widetilde\mu)+|\theta-\widetilde\theta|\bigr).
    \end{align*}
\end{enumerate}
\end{assumption}

\endgroup
 
We write $\bA=\{R,M,M_\ell,L_\ell,L_0,M_0,L_1,L_2,L_3\}$ for the bundle of constants in Assumption~\ref{ass:regularity} and $C=C(\bA)$ for any positive constant depending only on these parameters; in particular, they are independent of the embedding dimension $d$, the parameter dimension $p$, and the context length $n$.

For $\cV$ with the following form, we can obtain the parametric rate $n^{-1/2}$ in \Cref{thm:forward POC main} and \Cref{thm:OGD POC main}.
\begingroup
\renewcommand{\theassumption}{B}
\begin{assumption}\label{ass:kernel form V}
    The vector field $\cV(x,\mu;\theta)$ has the form
        \begin{align}
            \cV(x,\mu;\theta) = \cF\left(\int_{\RR^d} F(x,y;\theta) \ \d \mu(y)\right) + G(x;\theta),
        \end{align}
    where $F(x,y;\theta) : \RR^d \times \RR^d \times \RR^p \to \RR^m$ for some $m \in \ZZ_+$, $\cF:\RR^m \to \RR^d$, and $G(x;\theta) : \RR^d \times \RR^p \to \RR^d$. Also, there exists an $L_4 >0$, such that for all $x,y\in\RR^d$ and $\theta\in\RR^p$,
        \begin{align}
            |F(x,y;\theta)| \leq L_4(|x| +|y|+|\theta| +1), 
                \quad \|DF(x,y;\theta)\|\leq L_4, \quad \|D^2 F(x,y;\theta)\|\leq L_4,
        \end{align}
        and
        \begin{align}
            |G(x;\theta)| \leq L_4(|x|+|\theta| +1), 
                \quad \|DG(x;\theta)\|\leq L_4, \quad \|D^2 G(x;\theta)\|\leq L_4,
        \end{align}
    Here, $D,D^2$ denotes the Jacobian and the Hessian of $F,G$ with respect to $x,y,\theta$. Also, for all $w \in \RR^m$,
        \begin{align}
            |\cF(w)| \leq L_4(|w|+1), \quad \|D\cF(w)\| \leq L_4, \quad \|D^2\cF(w)\| \leq L_4.
        \end{align}
\end{assumption}
\endgroup
If the activation function in the MLP layer is ReLU, it is not $  C^2  $-regular as required by \Cref{ass:kernel form V}. Nevertheless, the theory developed in this paper also applies to this case. See \Cref{rmk:ReLU} for further discussion.

\paragraph{Remark on localization.}
It suffices for Assumptions~\eqref{ass:V}--\eqref{ass:derivative} and \Cref{ass:kernel form V} to hold on the compact sets furnished by~\eqref{ass:compact}: our proofs invoke $\cV$ and its derivatives only along the trajectories of the contextual flow~\eqref{eq:model intro} and along the OGD iterates~\eqref{eq:OGD}. We establish the requisite uniform-in-depth boundedness of both in Lemma~\ref{lem:bound theta_k} and in Theorem~\ref{thm:stability of contextual flow} below. In particular, when $\cV$ is a transformer block (a self-attention layer or an MLP), Assumptions~\eqref{ass:V}--\eqref{ass:derivative} and \Cref{ass:kernel form V} are derived from the single global hypothesis~\eqref{ass:compact} with constants that are dimension-free (Appendix~\ref{sec:verify transformer}).

\section{Conclusion}
We introduced \emph{contextual flow maps} as a unifying abstraction for analyzing transformers and developed a quantitative statistical theory in which the context length~$n$ is the central statistical resource. Within this framework, we established forward and backward propagation-of-chaos bounds at the Wasserstein rate~$n^{-1/d}$, uniformly along the depth of the network and, in the regularized regime, uniformly along the iterates of online gradient descent. To our knowledge, these are the first such guarantees for a deep stack of attention layers, where tokens lose their initial independence after the first layer through the shared velocity field. Under the additional structural condition that contextual measure $\mu$ is applied linearly when computing the vector field $\cV$---a condition satisfied by self-attention---we sharpen these bounds to the parametric rate $n^{-1/2}$, both for the distinguished token and along the OGD trajectory. Identifying weaker structural conditions under which the parametric rate persists is a natural direction for future work.

A central technical ingredient, of independent interest, is the Eulerian adjoint equation of Section~\ref{sec:adjoint}. In contrast to the classical neural-ODE adjoint, which features a single covector for the distinguished token, the contextual setting requires \emph{two} coupled adjoint variables: a covector $p_s$ tracking sensitivity to the token, and a scalar test field $\phi_s$ tracking sensitivity to the contextual measure. The $C^{1,1}$-regularity of $\phi_s$ is exactly what allows Wasserstein stability of the loss gradient, and we expect this two-variable construction to be useful for stability analysis of trained transformers and for the Riemannian extension to post-LayerNorm architectures.

\section*{Acknowledgment}

PR is partially supported by NSF grant DMS-2509011.

\newpage
\appendix

 
\section{Stability of the contextual flow: Proofs of Theorem~\ref{thm:stability flow main}~and~\ref{thm:forward POC main}}\label{sec:stability of contextual flow}
 
This appendix proves the stability bound for the forward flow stated in Theorem~\ref{thm:stability flow main}, establishes the corresponding concentration result for empirical contextual measures, and combines the two to deduce Theorem~\ref{thm:forward POC main} under \Cref{ass:regularity}.

\subsection{Stability of the contextual flow}
 
We restate Theorem~\ref{thm:stability flow main} in a slightly stronger form that also records the uniform-in-depth boundedness of $(x_s,\mu_s)$. Throughout the section, $\theta,\widetilde\theta:[0,1]\to\RR^p$ are two parameter paths, $\mu_0,\widetilde\mu_0\in\cP(\RR^d)$ are two initial measures, and $x_0,\widetilde x_0\in\RR^d$ are two initial tokens, all satisfying~\eqref{ass:compact}. The corresponding solutions to~\eqref{eq:model intro} are denoted $(x_s,\mu_s)$ and $(\widetilde x_s,\widetilde\mu_s)$.
 
\begin{theorem}[Stability of the contextual flow---precise version]\label{thm:stability of contextual flow}
Adopt Assumptions~\eqref{ass:compact} and~\eqref{ass:V}. There exists $C=C(\bA)$ such that, for all $s\in[0,1]$,
\begin{align}
    &\supp\mu_s\subseteq B_C(0),\qquad |x_s|\leq C,\label{eq:support mu}\\[2pt]
    &\bW_1(\mu_s,\widetilde\mu_s)\leq C\bigl(\bW_1(\mu_0,\widetilde\mu_0)+\|\theta-\widetilde\theta\|_{L^1}\bigr),\label{eq:bound mu}\\[2pt]
    &|x_s-\widetilde x_s|\leq C\bigl(\bW_1(\mu_0,\widetilde\mu_0)+|x_0-\widetilde x_0|+\|\theta-\widetilde\theta\|_{L^1}\bigr).\label{eq:bound xs}
\end{align}
\end{theorem}
 
The proof rests on the characteristic ODEs associated with~\eqref{eq:model intro}: let $\Phi_s,\widetilde\Phi_s:\RR^d\to\RR^d$ solve
\begin{equation}\label{eq:char flow}
    \begin{cases}
    \partial_s\Phi_s(x)=\cV(\Phi_s(x),\mu_s;\theta(s)),\\
    \Phi_0(x)=x,
    \end{cases}
    \qquad
    \begin{cases}
    \partial_s\widetilde\Phi_s(x)=\cV(\widetilde\Phi_s(x),\widetilde\mu_s;\widetilde\theta(s)),\\
    \widetilde\Phi_0(x)=x.
    \end{cases}
\end{equation}
Then $x_s=\Phi_s(x_0)$ and $\mu_s=(\Phi_s)_\#\mu_0$, and analogously for the tilded quantities.
 
\begin{proof}[Proof of~\eqref{eq:support mu}]
By the linear-growth bound in~\eqref{ass:V} and~\eqref{ass:compact},
\[
    \tfrac{\d}{\d s}|x_s|\leq|\partial_s\Phi_s(x_0)|=|\cV(x_s,\mu_s;\theta(s))|\leq M_0(|x_s|+M+1),
\]
and Gr\"onwall's inequality yields $|x_s|\leq(R+M+1)e^{M_0}+(M+1)$. The same argument applied uniformly over $x\in\supp\mu_0$ and the identity $\mu_s=(\Phi_s)_\#\mu_0$ give the support bound.
\end{proof}
 
\begin{proof}[Proof of~\eqref{eq:bound mu}]
Let $\gamma_0$ be any coupling of $\mu_0$ and $\widetilde\mu_0$. Then $(\Phi_s,\widetilde\Phi_s)_\#\gamma_0$ is a coupling of $\mu_s$ and $\widetilde\mu_s$, so
    \begin{align}\label{eq:W_1 <= joint flow}
        \bW_1(\mu_s, \widetilde \mu_s) \leq \int_{\RR^d \times \RR^d} | x - y | \ \d(\Phi_s ,\widetilde \Phi_s)_{\#} \gamma_0 (x,y) =  \int_{\RR^d \times \RR^d}| \Phi_s(x) - \widetilde \Phi_s(y) | \d \gamma_0 (x,y).
    \end{align}
Differentiating in $s$ and using~\eqref{eq:char flow} and~\eqref{ass:V},
    \begin{align*}
        \begin{split}
            &\frac{\d}{\d s}\int_{\RR^d \times \RR^d}| \Phi_s(x) - \widetilde \Phi_s(y) | \d \gamma_0 (x,y) \leq \int_{\RR^d \times \RR^d}|\partial_s \Phi_s(x) - \partial_s \widetilde \Phi_s(y) | \d \gamma_0 (x,y) 
            \\ &\leq L_0 \int_{\RR^d \times \RR^d}\left(| \Phi_s(x) - \widetilde \Phi_s(y) |+\bW_1(\mu_s,\widetilde \mu_s) + |\theta(s) - \widetilde \theta(s)| \right) \d \gamma_0 (x,y) 
            \\ &\leq 2 L_0 \int_{\RR^d \times \RR^d}| \Phi_s(x) - \widetilde \Phi_s(y) | \d \gamma_0 (x,y) + L_0  |\theta(s) - \widetilde \theta(s)|,
        \end{split}
    \end{align*}
where the last step uses~\eqref{eq:W_1 <= joint flow}. Gr\"onwall's inequality yields, for $s\in[0,1]$,
    \begin{align*}
        \int_{\RR^d \times \RR^d}| \Phi_s(x) - \widetilde \Phi_s(y) | \d \gamma_0 (x,y) \leq e^{2L_0s}\left( \int_{\RR^d \times \RR^d}| x - y | \d \gamma_0 (x,y) + L_0 \|\theta - \widetilde \theta\|_{L^1} \right).
    \end{align*}
Taking the infimum over couplings $\gamma_0$ gives~\eqref{eq:bound mu}.
\end{proof}
 
\begin{proof}[Proof of~\eqref{eq:bound xs}]
Since $|x_s-\widetilde x_s|=|\Phi_s(x_0)-\widetilde\Phi_s(\widetilde x_0)|$, applying~\eqref{ass:V} and~\eqref{eq:bound mu} gives
\[
    \tfrac{\d}{\d s}|x_s-\widetilde x_s|
    \leq L_0\bigl(|x_s-\widetilde x_s|+\bW_1(\mu_s,\widetilde\mu_s)+|\theta(s)-\widetilde\theta(s)|\bigr).
\]
Substituting~\eqref{eq:bound mu} into the right-hand side and applying Gr\"onwall yields~\eqref{eq:bound xs}.
\end{proof}
 
Theorem~\ref{thm:stability flow main} follows directly from the estimates \eqref{eq:bound mu}--\eqref{eq:bound xs} in Theorem~\ref{thm:stability of contextual flow}.

\subsection{Concentration of the empirical contextual flow}\label{sec:concentration contextual flow}
 
We now turn to the central random object. Let $\mu_0\in\cP(\RR^d)$ satisfy~\eqref{ass:compact}, draw $z^{(1)},\ldots,z^{(n)}\stackrel{\mathrm{iid}}{\sim}\mu_0$, and form the empirical measure
\begin{equation}\label{eq:empirical prompt measure}
    \widehat\mu_0\;\coloneqq\;\frac1n\sum_{i=1}^n\delta_{z^{(i)}}.
\end{equation}
We compare the contextual flow~\eqref{eq:model intro} starting from $\mu_0$ with the one starting from $\widehat\mu_0$, both evolved with the same parameter path $\theta$ and the same initial token $x_0$. The corresponding solutions are denoted $(x_s,\mu_s)$ and $(\widehat x_s,\widehat\mu_s)$.
 
\begin{theorem}[Concentration of the empirical contextual flow]\label{thm: concentration of empirical contextual flow}
Adopt Assumptions~\eqref{ass:compact} and~\eqref{ass:V}. Suppose $d\geq 3$. There exist constants $C,C_1,\alpha>0$ depending on $\bA$ such that
\begin{align}\label{eq:average concentration empirical contextual flow}
    \EE\Bigl[\sup_{s\in[0,1]}\bW_1(\mu_s,\widehat\mu_s)\Bigr]\;&\leq\; Cn^{-1/d},
    &
    \EE\Bigl[\sup_{s\in[0,1]}|x_s-\widehat x_s|\Bigr]\;&\leq\; Cn^{-1/d}.
\end{align}
For every $\varepsilon>0$,
\begin{align}
    \PP\Bigl(\bigl|\sup_{s}\bW_1(\mu_s,\widehat\mu_s)-\EE[\sup_s \bW_1(\mu_s,\widehat\mu_s)]\bigr|>\varepsilon\Bigr)
    &\leq 2\exp(-\alpha n\varepsilon^2),
    \label{eq:concentration empirical contextual flow 1}\\
    \PP\Bigl(\bigl|\sup_{s}|x_s-\widehat x_s|-\EE[\sup_s|x_s-\widehat x_s|]\bigr|>\varepsilon\Bigr)
    &\leq 2\exp(-\alpha n\varepsilon^2).
    \label{eq:concentration empirical contextual flow 2}
\end{align}
Consequently, for every $\beta>0$, with probability at least $1-4e^{-\beta}$,
\begin{equation}\label{eq:high prob concentration empirical contextual flow}
    \sup_{s\in[0,1]}\bW_1(\mu_s,\widehat\mu_s)\,+\,\sup_{s\in[0,1]}|x_s-\widehat x_s|
    \;\leq\;Cn^{-1/d}+C_1\beta^{1/2}n^{-1/2}.
\end{equation}
\end{theorem}
 
The proof rests on two classical facts that we record here for the reader's convenience.
 
\begin{lemma}[Wasserstein law of large numbers]\label{lem:wss-1 lln}
Let $d\geq 3$ and let $\mu_0\in\cP(\RR^d)$ with $\supp\mu_0\subseteq B_R(0)$. There exists $C=C(R)>0$ such that
\[
    \EE[\bW_1(\mu_0,\widehat\mu_0)]\leq Cn^{-1/d}.
\]
For $d=1$ the rate is $n^{-1/2}$; for $d=2$ it is $n^{-1/2}\log n$. If $\mu_0$ is supported on a $d'$-dimensional submanifold of $\RR^d$, the rate becomes $n^{-1/d'}$~\cite{WeeBac19}.
\end{lemma}
 
\begin{lemma}[McDiarmid's inequality]\label{lem:McDiarmid ineql}
Let $\xi_1,\ldots,\xi_N$ be independent random variables in $\RR^q$, and let $f:(\RR^q)^N\to\RR$ satisfy: for each $i$, there exists $c_i>0$ such that
\[
    |f(\ldots,x_i,\ldots)-f(\ldots,\widetilde x_i,\ldots)|\leq c_i
\]
whenever the two arguments differ only in the $i$-th coordinate. Then, for every $\varepsilon>0$,
\[
    \PP\bigl(|f(\xi_1,\ldots,\xi_N)-\EE f(\xi_1,\ldots,\xi_N)|\geq\varepsilon\bigr)
    \leq 2\exp\!\Bigl(-\tfrac{2\varepsilon^2}{\sum_{i=1}^N c_i^2}\Bigr).
\]
\end{lemma}
 
\begin{proof}[Proof of Theorem~\ref{thm: concentration of empirical contextual flow}]
Theorem~\ref{thm:stability of contextual flow}, applied with $\widetilde\mu_0=\widehat\mu_0$, $\widetilde x_0=x_0$, and $\widetilde\theta=\theta$, yields the deterministic bounds
\begin{equation}\label{eq:det bound emp flow}
    \sup_{s\in[0,1]}\bW_1(\mu_s,\widehat\mu_s)\,+\,\sup_{s\in[0,1]}|x_s-\widehat x_s|\;\leq\;C\,\bW_1(\mu_0,\widehat\mu_0).
\end{equation}
Taking expectations and applying Lemma~\ref{lem:wss-1 lln} gives~\eqref{eq:average concentration empirical contextual flow}.
 
For the concentration claims, view $f(z^{(1)},\ldots,z^{(n)})\coloneqq\sup_{s\in[0,1]}\bW_1(\mu_s,\widehat\mu_s)$ as a function of the i.i.d.\ sample. Replacing one sample $z^{(i)}$ by an independent copy $\widetilde z^{(i)}$ changes the empirical measure to $\widehat\mu_s'$. Because
    \begin{align*}
        \left| \sup_{s\in[0,1]}\bW_1(\mu_s,\widehat\mu_s) - \sup_{s\in[0,1]}\bW_1(\mu_s,\widehat\mu_s ') \right| \leq \sup_{s\in[0,1]} \left| \bW_1(\mu_s,\widehat\mu_s) - \bW_1(\mu_s,\widehat\mu_s ') \right| \leq  \sup_{s\in[0,1]} \bW_1(\widehat \mu_s,\widehat\mu_s '),
    \end{align*}
which is at most $C\bW_1(\widehat\mu_0,\widehat\mu_0')\leq C2R/n$ by \eqref{eq:bound mu}. So, the value of $f$ changes by at most $C\cdot 2R/n$. McDiarmid's inequality (Lemma~\ref{lem:McDiarmid ineql}) with $c_i=2RC/n$ gives $\sum c_i^2\leq 4R^2C^2/n$, and~\eqref{eq:concentration empirical contextual flow 1} follows with $\alpha=1/(2R^2C^2)$. The same argument applied to $|x_s-\widehat x_s|$ yields~\eqref{eq:concentration empirical contextual flow 2}, and combining mean and concentration gives~\eqref{eq:high prob concentration empirical contextual flow}.
\end{proof}
 
\subsection{Proof of Theorem~\ref{thm:forward POC main}}
 
\eqref{eq:forward POC bound} in Theorem~\ref{thm:forward POC main} follows immediately from Theorem~\ref{thm: concentration of empirical contextual flow} by choosing $\beta=n^{1-2/d}$. Indeed, $C_1\beta^{1/2}n^{-1/2}=C_1 n^{-1/d}$, so~\eqref{eq:high prob concentration empirical contextual flow} delivers
\[
    \sup_{s\in[0,1]}\bW_1(\mu_s,\widehat\mu_s)\,+\,\sup_{s\in[0,1]}|x_s-\widehat x_s|
    \;\leq\;(C+C_1)\,n^{-1/d}
\]
with probability at least $1-4\exp(-n^{1-2/d})$. The case $d=1,2$ is recovered from the corresponding statement of Lemma~\ref{lem:wss-1 lln}.

\eqref{eq:intro forward POC sharp bound} in Theorem~\ref{thm:forward POC main} follows from the later \Cref{thm:forward POC sharp} by choosing $\beta$ such that $2e^{-\beta} = \delta$.
\qed

 
\section{The adjoint system: Proofs of Theorem~\ref{thm:gradient adjoint}~and~\ref{thm:stability gradient main}}\label{sec:adjoint system}
 
This appendix derives the gradient formula~\eqref{eq:gradient} together with the adjoint equations~\eqref{eq:adjoint token}--\eqref{eq:adjoint measure} (Theorem~\ref{thm:gradient adjoint} in the body), establishes the well-posedness, regularity, and stability of the adjoint variables, and proves the stability of the loss gradient (Theorem~\ref{thm:stability gradient main} in the body).
 
\subsection{Eulerian derivation of the gradient formula}\label{sec:proof of main grad para}
 
We use the Eulerian Lagrangian approach: introduce two multipliers, $p_s:[0,1]\to\RR^d$ for the token equation and $\phi_s:[0,1]\times\RR^d\to\RR$ for the continuity equation, and choose them so that the variation of the Lagrangian with respect to $(x_s,\mu_s)$ vanishes identically. The remaining variation, with respect to $\theta$, yields~\eqref{eq:gradient}.
 
Fix $x_0\in\RR^d$, $y_0\in\RR^d$, an arbitrary continuous direction $\eta:[0,1]\to\RR^p$, and set $\theta^\varepsilon(s)\coloneqq\theta(s)+\varepsilon\eta(s)$. Let $(x_s^\varepsilon,\mu_s^\varepsilon)$ denote the corresponding solution of~\eqref{eq:model intro}. We start with the differentiability of $(x_s^\varepsilon,\mu_s^\varepsilon)$ in $\varepsilon$ at $\varepsilon=0$.

\paragraph{Differentiability with respect to $\varepsilon$.}
Let $\Phi_s^\varepsilon$ be the characteristic map associated with
$\mu_s^\varepsilon$:
\begin{equation}\label{eq:app_char}
    \begin{cases}
    \partial_s\Phi_s^\varepsilon(x)
    =\cV(\Phi_s^\varepsilon(x),\mu_s^\varepsilon;
        \theta^\varepsilon(s)),\\
    \Phi_0^\varepsilon(x)=x.
    \end{cases}
\end{equation}
Then $\mu_s^\varepsilon=(\Phi_s^\varepsilon)_\#\mu_0$ and $x_s^\varepsilon=\Phi_s^\varepsilon(x_0)$. 
Write $\Phi_s\coloneqq\Phi_s^0$ and note that $\mu_s=\mu_s^0$. 
For $|\varepsilon|$ sufficiently small, the paths $\{\theta_s ^\varepsilon\}$ are uniformly bounded, and the support estimate places all relevant perturbed characteristics in a common compact ball on which Assumptions~\eqref{ass:V} and~\eqref{ass:derivative} hold. In what follows, every $\sup_{a}$ is taken over for $a \in \supp\mu_0\cup\{x_0\}$.
We show $\Phi_s^\varepsilon$ is differentiable in $\varepsilon$ in the following.

Similar to the argument of Theorem~\ref{thm:stability of contextual flow}, applying Assumption~\eqref{ass:V} and Gr\"onwall's inequality also gives the preliminary stability estimate
\begin{equation}\label{eq:app_char_stab}
\sup_{s\in[0,1]}\sup_a
|\Phi_s^\varepsilon(a)-\Phi_s(a)|
\leq C|\varepsilon|\|\eta\|_{L^\infty},
\end{equation}
for $C = C(\bA)$. Now, let $U_s(a)$ be the unique $\RR^d$-valued solution, continuous in $a$ and absolutely continuous in $s$, of the following equation, understood for almost every $s$:
\begin{equation}\label{eq:app_diff_char}
\left\{
\begin{aligned}
\partial_sU_s(a)
&= D_x\cV(\Phi_s(a),\mu_s;\theta(s))U_s(a)
+ \int_{\RR^d}
\gradW_{\mu} \cV
[\Phi_s(a),\mu_s;\theta(s)](\Phi_s(b))
U_s(b)\,\d\mu_0(b) \\
&\quad
+ D_\theta\cV(\Phi_s(a),\mu_s;\theta(s))\eta(s), \\
U_0(a) &= 0.
\end{aligned}
\right.
\end{equation}
The form of \eqref{eq:app_diff_char} comes from taking derivative of $\Phi_s^\varepsilon(a)$ in $\varepsilon$ formally, whose existence, uniqueness, and regularity, follow directly from similar arguments in \Cref{prop:existence phi_s} and \Cref{prop: stability phi_s}. Assumption~\eqref{ass:derivative} implies that the first-order Taylor remainder of $\cV$ is quadratic.
Thus, subtracting the integral forms of the characteristic equations for both $\Phi^\varepsilon$ and $\Phi$, dividing by $\varepsilon$, applying the preceding Taylor expansion, then subtracting \eqref{eq:app_diff_char}, and applying \eqref{eq:app_char_stab} for the remainder term, gives, for every $s\in[0,1]$,
\[
\begin{aligned}
\sup_a\left|
\frac{\Phi_s^\varepsilon(a)-\Phi_s(a)}{\varepsilon}-U_s(a)
\right|
&\leq
C\int_0^s
\sup_a\left|
\frac{\Phi_r^\varepsilon(a)-\Phi_r(a)}{\varepsilon}-U_r(a)
\right|\,\d r
+C|\varepsilon|\|\eta\|_{L^\infty}^2.
\end{aligned}
\]
Therefore, another application of Gr\"onwall's inequality yields $\varepsilon\mapsto\Phi^\varepsilon$ is differentiable in the uniform path norm, with
\(
\left.\partial_\varepsilon\Phi_s^\varepsilon
\right|_{\varepsilon=0}=U_s.
\)

Consequently, for every $f\in C^1(\RR^d)$, the pushforward identity
and the mean-value formula give, uniformly in $s$,
\begin{equation*}
    \left\langle\sigma_s,f\right\rangle
    \coloneqq
    \left.\frac{\d}{\d\varepsilon}\right|_{\varepsilon=0}
        \int_{\RR^d}f(z)\,\d\mu_s^\varepsilon(z)
    =
    \int_{\RR^d}
        \nabla f(\Phi_s(a))\cdot U_s(a)
    \,\d\mu_0(a).
\end{equation*}
This formula defines $\sigma_s$ as a compactly supported distribution
of order at most one, with $\supp\sigma_s\subseteq\Phi_s(\supp\mu_0) \subseteq B_C(0)$ for $C= C(\bA)$ using \Cref{thm:stability of contextual flow}, and $\sigma_0=0$ because $\mu_0^\varepsilon=\mu_0$. Moreover, for this $C = C(\bA)$, let $f_s^\varepsilon\in C_c ^1(\RR^d)$ be a family of test function with $\supp f_s^\varepsilon \subseteq B_{2C}(0)$ and suppose further
that
\[
    \frac{f_s^\varepsilon(z)-f_s^0(z)}{\varepsilon}
    \longrightarrow
    \left.\partial_\varepsilon f_s^\varepsilon(z)
    \right|_{\varepsilon=0}
\]
uniformly in $(s,z)$ on this ball, with a jointly continuous limit.
Then,
\begin{equation*}
\left.\frac{\d}{\d\varepsilon}\right|_{\varepsilon=0}
        \int_{\RR^d}f_s^\varepsilon(z)\,\d\mu_s^\varepsilon(z)
    =
    \left\langle\sigma_s,f_s^0\right\rangle
    +\int_{\RR^d}
        \left.\partial_\varepsilon f_s^\varepsilon(z)
        \right|_{\varepsilon=0}
    \,\d\mu_s(z),
\end{equation*}
which will be frequently used below to differentiate the Lagrangian.

\paragraph{Weak measure-valued Lagrangian.}
For a smooth multiplier $\phi_s$ and an absolutely continuous
multiplier $p_s$, define the token and measure constraint terms by
\begin{align*}
    \mathsf T (\theta^\varepsilon)
    &\coloneqq
    \int_0^1p_s\cdot
        \bigl(\dot x_s^\varepsilon
        -\cV(x_s^\varepsilon,\mu_s^\varepsilon;
            \theta^\varepsilon(s))\bigr)\,\d s,\\
    \mathsf M(\theta^\varepsilon)
    &\coloneqq
    \int_{\RR^d}\phi_1(z)\,\d\mu_1^\varepsilon(z)
    -\int_{\RR^d}\phi_0(z)\,\d\mu_0(z)\\
    &\quad-\int_0^1\!\!\int_{\RR^d}
        \Bigl[
            \partial_s\phi_s(z)
            +\nabla\phi_s(z)\cdot
                \cV(z,\mu_s^\varepsilon;\theta^\varepsilon(s))
        \Bigr]
    \,\d\mu_s^\varepsilon(z)\,\d s.
\end{align*}
The identity $\mathsf M(\theta^\varepsilon)=0$ is exactly the weak
formulation of
\[
    \partial_s\mu_s^\varepsilon
    +\nabla\!\cdot\!
        \bigl(\mu_s^\varepsilon
        \cV(\,\cdot\,,\mu_s^\varepsilon;\theta^\varepsilon(s))\bigr)
    =0.
\]
It is meaningful for an arbitrary probability measure.  Compact support
removes boundary terms at infinity; equivalently, one may multiply
$\phi$ by a smooth cutoff that equals one on the common support.

Define the augmented Lagrangian
\[
    L(\theta^\varepsilon)
    \coloneqq
    \ell(x_1^\varepsilon,y_0)
    -\mathsf T(\theta^\varepsilon)
    -\mathsf M(\theta^\varepsilon).
\]
Since $(x_s^\varepsilon,\mu_s^\varepsilon)$ satisfies~\eqref{eq:model intro}, both constraints $\mathsf T(\theta^\varepsilon)$ and $\mathsf M(\theta^\varepsilon)$ vanish for every choice of $(p_s,\phi_s)$, and so $L(\theta^\varepsilon)=\cL(\theta^\varepsilon;\mu_0,x_0,y_0)$ for all $\varepsilon$.

\paragraph{Token term.}
Integration by parts in $s$, 
Using $x_0^\varepsilon=x_0$ 
and the measure chain rule
give
\begin{equation}\label{eq:term 1 simplified}
    \begin{aligned}
    \left.\frac{\d}{\d\varepsilon}\right|_{\varepsilon=0}
        \mathsf T(\theta^\varepsilon)
    &=
    p_1\cdot \partial_\varepsilon x_1^0
    -\int_0^1
        \bigl(\dot p_s+
        D_x\cV(x_s,\mu_s;\theta(s))^{\!\top}p_s\bigr)
        \cdot \partial_\varepsilon x_s^0 \,\d s\\
    &\quad-\int_0^1
        \left\langle\sigma_s,\,
            p_s\cdot
            \frac{\delta\cV}{\delta\mu}
            [x_s,\mu_s;\theta(s)](\,\cdot\,)
        \right\rangle\,\d s\\
    &\quad-\int_0^1
        p_s\cdot
        \bigl(D_\theta\cV(x_s,\mu_s;\theta(s))\eta(s)\bigr)
    \,\d s .
    \end{aligned}
\end{equation}

\paragraph{Measure term.}
Differentiating the weak measure constraint and using $\sigma_0=0$
gives
\begin{equation}\label{eq:term 2 simplified}
    \begin{aligned}
    \left.\frac{\d}{\d\varepsilon}\right|_{\varepsilon=0}
        \mathsf M(\theta^\varepsilon)
    &=
    \langle\sigma_1,\phi_1\rangle\\
    &\quad-\int_0^1
        \bigg\langle\sigma_s,\,
            \partial_s\phi_s(\,\cdot\,)
            +\nabla\phi_s(\,\cdot\,)\cdot
                \cV(\,\cdot\,,\mu_s;\theta(s))\\
    &\hspace{8em}
            +\int_{\RR^d}
                \nabla\phi_s(\xi)\cdot
                \frac{\delta\cV}{\delta\mu}
                [\xi,\mu_s;\theta(s)](\,\cdot\,)
            \,\d\mu_s(\xi)
        \bigg\rangle\,\d s\\
    &\quad-\int_0^1\!\!\int_{\RR^d}
        \nabla\phi_s(z)\cdot
        \bigl(D_\theta\cV(z,\mu_s;\theta(s))\eta(s)\bigr)
    \,\d\mu_s(z)\,\d s .
    \end{aligned}
\end{equation}
For smooth $\phi$, every pairing in
\eqref{eq:term 2 simplified} is well-defined through the explicit
action of $\sigma_s$ on $C^1$ test functions.  
 
\paragraph{Combination.}
Adding $\nabla_x\ell(x_1^0,y_0)\!\cdot\!\partial_\varepsilon x_1^0$ to~\eqref{eq:term 1 simplified} and~\eqref{eq:term 2 simplified}, and choosing $(p_s,\phi_s)$ so that the coefficients of $\partial_\varepsilon x_1^0$, $\partial_\varepsilon x_s^0$, 
$\sigma_1$, $\sigma_s$
all vanish, leads precisely to the adjoint system~\eqref{eq:adjoint token}--\eqref{eq:adjoint measure} of the body. The remaining terms involve only $\eta$ and yield, after collecting:
\[
    \frac{\d}{\d\varepsilon}\Big|_{\varepsilon=0}\!\!\cL(\theta^\varepsilon;\mu_0,x_0,y_0)
    \;=\;\int_0^1\Bigl[(D_\theta\cV)^{\!\top}p_s+\!\!\int_{\RR^d}\!(D_\theta\cV(z,\mu_s;\theta(s)))^{\!\top}\nabla\phi_s(z)\,\d\mu_s(z)\Bigr]\!\cdot\!\eta(s)\,\d s.
\]
Since $\eta$ is arbitrary, this is the formula~\eqref{eq:gradient}. \qed

\paragraph{Riemannian extension.}
The argument extends without modification when $(x_s,\mu_s)$ is constrained to a Riemannian submanifold $\bM\subseteq\RR^d$. The Lagrangian is unchanged (with the inner product on $\bM$ in place of the Euclidean one), the multiplier $p_s$ becomes a tangent vector in $T_{x_s}\bM$, and the spatial Jacobian $D_x\cV$ is replaced by the covariant derivative $\nabla^{\bM}\cV$. The resulting backward equation for $p_s$ reads
\[
    \dot p_s\;=\;-\bigl[\nabla^{\bM}\cV(x_s,\mu_s;\theta(s))\bigr]^{*}p_s,
\]
with $[\nabla^{\bM}\cV]^*$ the metric dual of the $(1,1)$-tensor $\nabla^{\bM}\cV$ at $x_s$; the equation for $\phi_s$ is unchanged in form once the inner products are reinterpreted by the Riemannian metric on $\bM$.

\subsection{Token adjoint: well-posedness and stability}
 
We now establish the bounds on $p_s$ and its dependence on the data. Existence and uniqueness of the solution to~\eqref{eq:adjoint token} follow from classical linear ODE theory.
 
\begin{proposition}[Bounds and stability for $p_s$]\label{prop: bound p}
Adopt the assumptions of Theorem~\ref{thm:gradient adjoint}. There exist constants $C_1,C_2$ depending on $\bA$ such that, with $p_s$ and $\widetilde p_s$ the solutions of~\eqref{eq:adjoint token} associated with $(x_s,\mu_s,\theta)$ and $(\widetilde x_s,\widetilde\mu_s,\widetilde\theta)$ respectively,
\begin{align}
    &\sup_{s\in[0,1]}|p_s|\;\leq\;C_1,\qquad\sup_{s\in[0,1]}|\widetilde p_s|\;\leq\;C_1,\label{eq:ps}\\[2pt]
    &\sup_{s\in[0,1]}|p_s-\widetilde p_s|\;\leq\;C_2\,\cK,\qquad\cK\coloneqq \bW_1(\mu_0,\widetilde\mu_0)+|x_0-\widetilde x_0|+|y_0-\widetilde y_0|+\|\theta-\widetilde\theta\|_{L^1}.\label{eq:ps stability}
\end{align}
\end{proposition}
 
\begin{proof}
By~\eqref{eq:adjoint token} and~\eqref{ass:derivative}, $\tfrac{\d}{\d s}|p_s|\leq L_0|p_s|$. The terminal condition $|p_1|\leq M_\ell(2C+1)$ (from~\eqref{ass:l} and the bounds in Theorem~\ref{thm:stability of contextual flow}) and Gr\"onwall's inequality yield~\eqref{eq:ps}. For~\eqref{eq:ps stability}, note first that
\[
    |p_1-\widetilde p_1|\;\leq\;L_\ell\bigl(|x_1-\widetilde x_1|+|y_0-\widetilde y_0|\bigr).
\]
Subtracting the two adjoint equations and using~\eqref{ass:V},~\eqref{ass:derivative}, and~\eqref{eq:ps},
\begin{align*}
    \tfrac{\d}{\d s}|p_s-\widetilde p_s|
    &\leq\bigl|D_x\cV(x_s,\mu_s;\theta(s))^{\!\top}p_s-D_x\cV(\widetilde x_s,\widetilde\mu_s;\widetilde\theta(s))^{\!\top}\widetilde p_s\bigr|\\
    &\leq L_1\bigl(|x_s-\widetilde x_s|+\bW_1(\mu_s,\widetilde\mu_s)+|\theta(s)-\widetilde\theta(s)|\bigr)\,C_1+L_0|p_s-\widetilde p_s|.
\end{align*}
Substituting the bounds of Theorem~\ref{thm:stability of contextual flow} and applying Gr\"onwall yields~\eqref{eq:ps stability}.
\end{proof}
 
\subsection{Measure adjoint: well-posedness, regularity, and stability}
 
The well-posedness of~\eqref{eq:adjoint measure} is more delicate than~\eqref{eq:adjoint token} because the right-hand side depends on $\nabla\phi_s$ in a nonlocal way. We resolve this by a contraction argument in a weighted $C^1$ space.
 
\begin{proposition}[Existence and regularity of $\phi_s$]\label{prop:existence phi_s}
Adopt the assumptions of Theorem~\ref{thm:gradient adjoint}. There exists a unique solution $\phi_s$ of~\eqref{eq:adjoint measure} in the class $\{\phi\in C([0,1];C^1(\RR^d)):\phi(s,\cdot)\in C^{1,1}(\RR^d)\}$. There exists $C=C(\bA)$ such that, for all $s\in[0,1]$ and $z,\widetilde z\in\RR^d$,
\[
    |\phi_s(z)|\leq C(|z|+1),\qquad|\nabla\phi_s(z)|\leq C,\qquad|\nabla\phi_s(z)-\nabla\phi_s(\widetilde z)|\leq C|z-\widetilde z|.
\]
\end{proposition}
 
\begin{proof}
Define the Banach space
\begin{equation}\label{eq:banach X}
    X\coloneqq\bigl\{\psi:[0,1]\times\RR^d\to\RR\,:\,\psi(s,\cdot)\in C^1(\RR^d),\,\|\psi\|_X<\infty\bigr\},
\end{equation}
with norm
\begin{align}\label{eq:banach X norm}
            \|\psi\|_X \coloneqq \sup_{s\in[0,1], \ z\in \RR^d} e^{\alpha(s-1)} \left( \frac{|\psi(s,z)| }{|z| + 1}+ |\nabla_z \psi(s,z)| \right),
        \end{align}
where $\alpha = \alpha(\bA)>0$ will be chosen below. For $\psi\in X$, set
\begin{equation}\label{eq:f(s,z;psi)}
    f(s,z;\psi)\coloneqq -p_s\!\cdot\!\tfrac{\delta\cV}{\delta\mu}[x_s,\mu_s;\theta(s)](z)
    -\int_{\RR^d}\!\nabla\psi_s(\xi)\!\cdot\!\tfrac{\delta\cV}{\delta\mu}[\xi,\mu_s;\theta(s)](z)\,\d\mu_s(\xi).
\end{equation}
By~\eqref{ass:V} and Proposition~\ref{prop: bound p}, there exists $C=C(\bA)$ such that
\begin{equation}\label{eq:X norm f}
    e^{\alpha(s-1)}|f(s,z;\psi)|\leq C(|z|+1)\bigl(\|\psi\|_X+e^{\alpha(s-1)}\bigr),
    \ 
    e^{\alpha(s-1)}|\nabla_z f(s,z;\psi)|\leq C\bigl(\|\psi\|_X+e^{\alpha(s-1)}\bigr).
\end{equation}
Let $\Psi_\tau(z)$ be the solution from depth $s$ of $\partial_\tau\Psi_\tau=\cV(\Psi_\tau,\mu_\tau;\theta(\tau))$ with $\Psi_s(z)=z$, $\forall z \in \RR^d$. Then, the transport equation
\begin{equation}\label{eq:linearization of adjoint function}
    \partial_s u(s,z)+\nabla_z u(s,z)\cdot\cV(z,\mu_s;\theta(s))=f(s,z;\psi),\qquad u(1,z)=0
\end{equation}
admits the explicit solution
\begin{equation}\label{eq: u contraction map}
    u(s,z)\;=\;-\int_s^1 f(\tau,\Psi_\tau(z);\psi)\,\d\tau.
\end{equation}
 We show that this $u(s,z) \in X$. First, by Assumption~\eqref{ass:V}, 
        \begin{align*}
            \partial_\tau |\Psi_\tau(z)| \leq |\partial_\tau  \Psi_\tau(z)| \leq M_0 (|\Psi_\tau(z)|+M).
        \end{align*}
    A Gr{\"o}nwall argument yields that $|\Psi_\tau(z)| \leq Me^{M_0} (|z|+1)$ for any $\tau \in [s,1]$ and any $z \in \RR^d$. Similarly, because
        \begin{align}\label{eq:Dz Psi}
            \partial_\tau D_z\Psi_\tau(z) = D_x\cV(\Psi_\tau(z),\mu_\tau;\theta(\tau)) D_z\Psi_\tau(z) , \quad D_z \Psi_s(z) = \Id,
        \end{align}
    We see that by~\eqref{ass:V},
        \begin{align*}
            \partial_\tau \|D_z\Psi_\tau(z)\| \leq L_0 \|D_z\Psi_\tau(z)\|.
        \end{align*}
By Gr\"onwall, $|\Psi_\tau(z)|\leq Me^{M_0}(|z|+1)$ and $\|D_z\Psi_\tau(z)\|\leq e^{L_0}$ uniformly on $[s,1]\times\RR^d$. Setting $s = \tau$ in~\eqref{eq:X norm f} and combining with \eqref{eq: u contraction map},
\begin{align*}
            \begin{split}
                \|u\|_X &\leq C\left(\|\psi\|_X \sup_{s \in [0,1]}e^{\alpha(s-1)}\int_s ^1 e^{\alpha(1-\tau)} \ \d \tau + 1\right)  
                \\  &= C\left(\frac{1}{\alpha}\|\psi\|_X \sup_{s \in [0,1]}(1-e^{\alpha(s-1)}) + 1\right) \leq C\left(\frac{1}{\alpha}\|\psi\|_X + 1\right) ,
            \end{split}
        \end{align*}
Choosing $\alpha=4C(C+1)$, the map $\psi\mapsto u$ leaves the ball $\{\psi\in X:\|\psi\|_X\leq C+1\}$ invariant.

 For any two $\psi_1,\psi_2 \in X$, we can similarly obtain $u_1,u_2$. By the linearity of $\psi$ in the function $f(s,z;\psi)$, we see that $f(s,z;\psi_1 - \psi_2)$ is homogeneous with respect to $\psi_1 - \psi_2$. We can thus similarly show that
        \begin{align*}
            \|u_1 - u_2\|_X \leq \frac{C}{\alpha} \|\psi_1 - \psi_2\|_X = \frac{1}{4} \|\psi_1 - \psi_2\|_X.
        \end{align*}
    Thus, the map from $\psi$ to $u$ is a contraction. By the Banach fixed-point theorem, there is a unique fixed point $\phi(s,z) \in X$ which satisfies 
        \begin{align}\label{e:phi X norm bound}
            \|\phi\|_X \leq C+1
        \end{align}
    and solves \eqref{eq:linearization of adjoint function} for $u = \psi = \phi$. This $\phi(s,z)$ is exactly the unique solution to \eqref{eq:adjoint measure}.
 
For the Lipschitz bound on $\nabla\phi_s$, differentiate~\eqref{eq: u contraction map} in $z$:
\[
    \nabla_z\phi(s,z)
    \;=\;-\int_s^1(D_z\Psi_\tau(z))^{\!\top}\bigl(\nabla_z f(\tau,\Psi_\tau(z);\phi)\bigr)\,\d\tau.
\]
We first notice that $D_z\Psi_\tau(z)$ is uniformly Lipschitz continuous, which follows from a direct Gr\"onwall argument for $\|D_z\Psi_\tau(z_1) - D_z\Psi_\tau(z_2)\|, \forall z_1,z_2 \in \RR^d$. Indeed, by \eqref{eq:Dz Psi},   
    \begin{align*}
       \begin{split}
            \partial_\tau \| D_z\Psi_\tau(z_1) - D_z\Psi_\tau(z_2)\| &\leq \|D_x\cV(\Psi_\tau(z_1),\mu_\tau;\theta(\tau))\| \cdot \|D_z\Psi_\tau(z_1) -  D_z\Psi_\tau(z_2)\| 
            \\ &\quad + \|D_x\cV(\Psi_\tau(z_1),\mu_\tau;\theta(\tau)) - D_x\cV(\Psi_\tau(z_2),\mu_\tau;\theta(\tau))\| \cdot \| D_z\Psi_\tau(z_2)\| 
            \\  & \leq  L_0 \|D_z\Psi_\tau(z_1) -  D_z\Psi_\tau(z_2)\| + L_1 e^{2L_0} |z_1- z_2| ,
       \end{split}
    \end{align*}
where the last inequality follows from the assumptions \eqref{ass:V} and  \eqref{ass:derivative} on $D_x\cV$, together with the previous proved fact that $\|D_z\Psi_\tau(z)\|\leq e^{L_0}$. A Gr\"onwall type argument then shows that $D_z\Psi_\tau(z)$ is Lipschitz with constant depending only on $\bA$. 

Next, using the explicit form~\eqref{eq:f(s,z;psi)},
\[
    (\nabla_z f(s,z;\phi))^{\!\top}=-p_s\!\cdot\!(\gradW_{\mu} \cV)[x_s,\mu_s;\theta(s)](z)-\int_{\RR^d}\!\nabla\phi_s(\xi)\!\cdot\!(\gradW_{\mu} \cV)[\xi,\mu_s;\theta(s)](z)\,\d\mu_s(\xi).
\]
By the Lipschitz assumption~\eqref{ass:derivative} on $\gradW_{\mu} \cV$ with respect to the $z$-variable, and the fact that $|\nabla\phi_s(\xi)| \leq \|\phi\|_X \leq C+1$ in \eqref{e:phi X norm bound}, we conclude that $\nabla_z f(s,z;\phi)$ is also Lipschitz continuous in $z$ with constant depending only on $\bA$. Thus, together with the uniform Lipschitz continuity of $D_z\Psi_\tau(z)$ showed above, $\nabla_z\phi(s,\cdot)$ is uniformly Lipschitz in $z$, with constant depending only on $\bA$.

\end{proof}

We turn to stability. Let $\theta,\widetilde\theta$, $\mu_0,\widetilde\mu_0$, $x_0,\widetilde x_0$, $y_0,\widetilde y_0$ all satisfy~\eqref{ass:compact}, with corresponding solutions $(\mu_s,p_s,\phi_s)$ and $(\widetilde\mu_s,\widetilde p_s,\widetilde\phi_s)$.
 
\begin{proposition}[Stability of $\phi_s$]\label{prop: stability phi_s}
Adopt the assumptions of Theorem~\ref{thm:gradient adjoint}. There exist $C_1,C_2$ depending on $\bA$ such that, for all $s\in[0,1]$ and $z\in\RR^d$,
\[
    |\phi_s(z)-\widetilde\phi_s(z)|\leq C_1\,\cK\,(|z|+1),
    \qquad
    |\nabla\phi_s(z)-\nabla\widetilde\phi_s(z)|\leq C_2\,\cK,
\]
where $\cK=\bW_1(\mu_0,\widetilde\mu_0)+|x_0-\widetilde x_0|+|y_0-\widetilde y_0|+\|\theta-\widetilde\theta\|_{L^1}$.
\end{proposition}
 
\begin{proof}
Set $\Gamma_s(z)\coloneqq \widetilde \phi_s(z)-\phi_s(z)$. Subtracting the two equations~\eqref{eq:adjoint measure} and rearranging,
\[
    \partial_s\Gamma_s(z)+\nabla_z(\Gamma_s(z))\!\cdot\!\cV(z,\widetilde \mu_s;\widetilde \theta(s))+\!\!\int_{\RR^d}\!\nabla(\Gamma_s(\xi))\!\cdot\!\tfrac{\delta\cV}{\delta\mu}[\xi,\widetilde \mu_s;\widetilde \theta(s)](z)\d\widetilde \mu_s(\xi)=g(s,z),
\]
with $\Gamma_1\equiv 0$ and source term
\begin{align*}
    g(s,z)
    &=-\nabla \phi_s(z)\!\cdot\!\bigl(\cV(z,\widetilde\mu_s;\widetilde\theta(s))-\cV(z,\mu_s;\theta(s))\bigr)\\
    &\quad-\bigl(\widetilde p_s\!\cdot\!\tfrac{\delta\cV}{\delta\mu}[\widetilde x_s,\widetilde\mu_s;\widetilde\theta(s)](z)- p_s\!\cdot\!\tfrac{\delta\cV}{\delta\mu}[x_s,\mu_s;\theta(s)](z)\bigr)\\
    &\quad-\!\int_{\RR^d}\!\nabla\phi_s(\xi)\!\cdot\!\bigl(\tfrac{\delta\cV}{\delta\mu}[\xi,\widetilde\mu_s;\widetilde\theta(s)](z)\,\d\widetilde\mu_s(\xi) - \tfrac{\delta\cV}{\delta\mu}[\xi,\mu_s;\theta(s)](z)\,\d\mu_s(\xi)\bigr).
\end{align*}
Combining the Lipschitz bounds in~\eqref{ass:V},~\eqref{ass:derivative}, Theorem~\ref{thm:stability of contextual flow}, Proposition~\ref{prop: bound p}, and Proposition~\ref{prop:existence phi_s}, we obtain
\begin{equation}\label{eq:X norm g}
    |g(s,z)|\leq C(|z|+1)\bigl(\cK+|\theta(s)-\widetilde\theta(s)|\bigr),
    \qquad
    |\nabla_z g(s,z)|\leq C\bigl(\cK+|\theta(s)-\widetilde\theta(s)|\bigr).
\end{equation}
The argument of Proposition~\ref{prop:existence phi_s}, applied to $\Gamma_s$ in the same Banach space $X$ but with the inhomogeneous source $g$, gives
\[
    \|\Gamma\|_X\leq C\Bigl(\tfrac1\alpha\|\Gamma\|_X+\cK\Bigr),
\]
and choosing $\alpha=4C$ yields $\|\Gamma\|_X\leq 4C\,\cK$, which is the stated bound.
\end{proof}
 
\subsection{Stability of the loss gradient}
 
We are now in position to prove Theorem~\ref{thm:stability gradient main}, restated below in its precise form.
 
\begin{theorem}[Stability of the loss gradient]\label{thm:stability of loss gradient}
Adopt the assumptions of Theorem~\ref{thm:gradient adjoint}. There exists $C=C(\bA)$ such that, for all $s\in[0,1]$,
\begin{align*}
    &\Bigl|\tfrac{\delta\cL(\theta;\mu_0,x_0,y_0)}{\delta\theta}(s)-\tfrac{\delta\cL(\widetilde\theta;\widetilde\mu_0,\widetilde x_0,\widetilde y_0)}{\delta\theta}(s)\Bigr|\\
    &\quad\leq C\bigl(\bW_1(\mu_0,\widetilde\mu_0)+|x_0-\widetilde x_0|+|y_0-\widetilde y_0|+\|\theta-\widetilde\theta\|_{L^1}+|\theta(s)-\widetilde\theta(s)|\bigr).
\end{align*}
\end{theorem}
 
\begin{proof}
Apply the triangle inequality to the gradient formula~\eqref{eq:gradient} and bound each piece using:
\begin{itemize}
    \item the Lipschitz bound on $D_\theta\cV$ in $(x,\mu,\theta)$ from~\eqref{ass:derivative},
    \item the bounds on $|p_s|$ and $|p_s-\widetilde p_s|$ from Proposition~\ref{prop: bound p},
    \item the bounds on $|\nabla\phi_s|$ and $|\nabla\phi_s-\nabla\widetilde\phi_s|$ from Propositions~\ref{prop:existence phi_s} and~\ref{prop: stability phi_s},
    \item the bounds on $|x_s-\widetilde x_s|$ and $\bW_1(\mu_s,\widetilde\mu_s)$ from Theorem~\ref{thm:stability of contextual flow}.
\end{itemize}
Each contribution is linear in $\cK+|\theta(s)-\widetilde\theta(s)|$ with $\cK$ defined in \Cref{prop: stability phi_s}, and the conclusion follows.
\end{proof}

 
\section{Non-asymptotic OGD analysis: proof of Theorem~\ref{thm:OGD POC main}}\label{sec:OGD stability}
 
This appendix proves Theorem~\ref{thm:OGD POC main} under \Cref{ass:regularity} via three intermediate results: a uniform-in-iteration boundedness lemma for the OGD trajectory, a deterministic stability bound with respect to data and initialization, and a McDiarmid-type concentration bound for the empirical OGD trajectory.
 
\subsection{Uniform boundedness of the OGD iterates}
 
The first lemma ensures that all results of Appendices~\ref{sec:stability of contextual flow}--\ref{sec:adjoint system} apply along the OGD trajectory~\eqref{eq:OGD}, with~\eqref{ass:compact} replaced by the slightly enlarged bound $\|\theta\|_{L^\infty}\leq M+3$.
 
\begin{lemma}[Boundedness of OGD]\label{lem:bound theta_k}
Adopt the assumptions of Theorem~\ref{thm:gradient adjoint}. There exists $C=C(\bA)$ such that, for any sequence of training inputs $\{(\mu_0^k,x_0^k,y_0^k)\}_{k\in\NN}$ satisfying~\eqref{ass:compact} and any initialization $\theta_0$ with $\|\theta_0\|_{L^\infty}\leq M$, the OGD iterates~\eqref{eq:OGD} satisfy:
\begin{enumerate}[label=\emph{(\roman*)}]
    \item For every $\eta>0$, $\lambda\in[0,\eta^{-1})$, and $k\leq(\log(1+C\eta))^{-1}$,
    \begin{align}\label{eq:finite-time bound OGD}
            \| \theta_k  \|_{L^{\infty}} \leq \|\theta_0\|_{L^{\infty}} + 3.
        \end{align}
    \item If $\eta^{-1}>C$ and $\lambda\in(C,\eta^{-1})$, then for every $k\in\NN$,
    \begin{align}\label{eq:long-time bound OGD}
            \| \theta_k \|_{L^{\infty}} \leq  \|\theta_0\|_{L^{\infty}} + 1.
        \end{align}
\end{enumerate}
\end{lemma}
 
\begin{proof}
By Theorem~\ref{thm:gradient adjoint} and the boundedness of $D_\theta\cV$ in~\eqref{ass:derivative}, $\|\delta\cL/\delta\theta\|_{L^\infty}\leq C$ for some $C=C(\bA)$. Iterating~\eqref{eq:OGD} gives
\[
    \|\theta_k\|_{L^\infty}\leq(1-\eta\lambda)^k\|\theta_0\|_{L^\infty}+\eta C\sum_{i=0}^{k-1}(1-\eta\lambda)^i.
\]
Each term in the geometric sum can be bounded by $(1+C\eta)^i$, so the second term is bounded by $(1+C\eta)^k =e< 3$ when $k\leq (\log(1+C\eta))^{-1}$, giving (i); when $\lambda>C$, it is bounded by $C/\lambda<1$ for every $k$, giving (ii).
\end{proof}

\subsection{Stability of the OGD iterates}
 
Let $\{(\mu_0^k,x_0^k,y_0^k)\}_{k\in\NN}$ and $\{(\widetilde\mu_0^k,\widetilde x_0^k,\widetilde y_0^k)\}_{k\in\NN}$ be two streams of training data, and $\theta_0,\widetilde\theta_0$ two initializations, all satisfying~\eqref{ass:compact}. Let $\{\theta_k\}$ and $\{\widetilde\theta_k\}$ be the corresponding OGD trajectories, and define the data discrepancy
\begin{equation}\label{eq:OGD data error}
    \cE_k\;\coloneqq\;\max_{\ell\leq k-1}\bigl(\bW_1(\mu_0^\ell,\widetilde\mu_0^\ell)+|x_0^\ell-\widetilde x_0^\ell|+|y_0^\ell-\widetilde y_0^\ell|\bigr).
\end{equation}
 
\begin{theorem}[Finite- and long-horizon stability of OGD]\label{thm:stability OGD}
Adopt the assumptions of Theorem~\ref{thm:gradient adjoint}. There exist $C=C(\bA), \eta_c = \eta_c(\bA), \lambda_c = \lambda_c(\bA) < \eta_c ^{-1}$ such that:
\begin{enumerate}[label=\emph{(\roman*)}]
    \item For every $\eta>0$, $\lambda\in[0,\eta^{-1})$, and $k\leq(\log(1+C\eta))^{-1}$,
    \begin{equation}\label{eq:finite-time stability OGD}
        \|\theta_k-\widetilde\theta_k\|_{L^\infty}\;\leq\;3\bigl(\|\theta_0-\widetilde\theta_0\|_{L^\infty}+\cE_k\bigr).
    \end{equation}
    \item For every $\eta<\eta_c$, $\lambda\in(\lambda_c,\eta^{-1})$, and every $k\in\NN$,
    \begin{equation}\label{eq:long-time stability OGD}
        \|\theta_k-\widetilde\theta_k\|_{L^\infty}\;\leq\;\|\theta_0-\widetilde\theta_0\|_{L^\infty}+\cE_k.
    \end{equation}
\end{enumerate}
\end{theorem}
 
\begin{proof}
Set $\Delta_l\coloneqq\|\theta_l-\widetilde\theta_l\|_{L^\infty}$. Subtracting the two OGD updates and applying Theorem~\ref{thm:stability of loss gradient},
\begin{equation}\label{eq:ineql ogd error}
    \Delta_{l+1}\leq(1-\eta\lambda)\Delta_l+\eta C\bigl(\Delta_l+\bW_1(\mu_0^l,\widetilde\mu_0^l)+|x_0^l-\widetilde x_0^l|+|y_0^l-\widetilde y_0^l|\bigr)
    \leq(1-\eta(\lambda-C))\Delta_l+\eta C\,\cE_k.
\end{equation}
 
\emph{Statement (i).} For $\lambda\geq 0$, $1-\eta(\lambda-C)\leq 1+\eta C$, and iterating yields
\[
    \Delta_k\leq(1+\eta C)^k\Delta_0+\eta C\,\cE_k\sum_{i=0}^{k-1}(1+\eta C)^i\leq(1+\eta C)^k(\Delta_0+\cE_k).
\]
For $k\leq(\log(1+\eta C))^{-1}$, $(1+\eta C)^k\leq e<3$, giving~\eqref{eq:finite-time stability OGD}.
 
\emph{Statement (ii).} Choose $\lambda_c=2C$ and $\eta_c=(2C)^{-1}$. Then $1-\eta(\lambda-C)\in(0,1)$, and iterating~\eqref{eq:ineql ogd error} gives, for every $k\in\NN$,
\[
    \Delta_k\leq(1-\eta C)^k\Delta_0+\eta C\,\cE_k\sum_{i=0}^{k-1}(1-\eta C)^i\leq\Delta_0+\cE_k.
\]
\end{proof}
 
\subsection{Empirical OGD: mean and concentration bounds}
 
We now specialize Theorem~\ref{thm:stability OGD} to the case in which $\widetilde\theta_k=\widehat\theta_k$ is driven by the empirical contextual measure
\begin{equation}\label{eq:empirical input data measure}
    \widehat\mu_0^k\;\coloneqq\;\frac{1}{n}\sum_{i=1}^n\delta_{z^{(k,i)}},\qquad z^{(k,i)}\stackrel{\mathrm{iid}}{\sim}\mu_0^k,
\end{equation}
with the same initialization $\theta_0$, the same $(x_0^k,y_0^k)$, and the same update rule.
 
\begin{theorem}[Mean bound for the empirical OGD]\label{thm:stability empirical OGD}
Adopt the assumptions of Theorem~\ref{thm:gradient adjoint}, and let $d\geq 3$. There exist $C=C(\bA), \eta_c = \eta_c(\bA),\lambda_c = \lambda_c(\bA) < \eta_c(\bA) ^{-1}$ such that:
\begin{enumerate}[label=\emph{(\roman*)}]
    \item For $\eta>0$, $\lambda\in[0,\eta^{-1})$, we let $k_c \coloneq (\log(1+C\eta))^{-1}$. Then,
    \begin{equation}\label{eq:finite-time stability empirical OGD}
        \EE\sup_{k \leq k_c}\|\theta_k-\widehat\theta_k\|_{L^\infty}\leq 3n^{-1/d}.
    \end{equation}
    \item For $\eta<\eta_c$ and $\lambda\in(\lambda_c,\eta^{-1})$, every $k \in \NN$ satisfies
    \begin{equation}\label{eq:long-time stability empirical OGD}
        \EE\|\theta_k-\widehat\theta_k\|_{L^\infty}\leq n^{-1/d}.
    \end{equation}
\end{enumerate}
If $\mu_0^k$ is supported on a $d'$-dimensional submanifold of $\RR^d$, the rate $n^{-1/d}$ is replaced by $n^{-1/d'}$.
\end{theorem}
 
\begin{proof}
Set $\Delta_l=\|\theta_l-\widehat\theta_l\|_{L^\infty}$ and apply Theorem~\ref{thm:stability of loss gradient} as in~\eqref{eq:ineql ogd error}. Since the only data discrepancy is the contextual measure,
\begin{equation}\label{eq:ineql empirical ogd error}
    \Delta_{l+1}\leq(1-\eta(\lambda-C))\Delta_l+\eta C\,\bW_1(\mu_0^l,\widehat\mu_0^l).
\end{equation}
Notice that $\Delta_0 = 0$. 

\emph{Statement (i).} For $\lambda\geq 0$, $1-\eta(\lambda-C)\leq 1+\eta C$, and iterating yields
\[
    \sup_{k \leq k_c} \Delta_k\leq\eta C\sum_{l=0}^{k_c-1}(1+\eta C)^{k_c-1-l} \bW_1(\mu_0^l,\widehat\mu_0^l).
\]
Take expectations and apply Lemma~\ref{lem:wss-1 lln} to each $\EE \bW_1(\mu_0^l,\widehat\mu_0^l)$, we find that 
    \begin{align*}
        \EE  \sup_{k \leq k_c} \Delta_k \leq \eta C n^{-1/d} \sum_{l=0}^{k_c-1}(1+\eta C)^{k_c-1-l} \leq (1+\eta C)^{k_c} n^{-1/d}
    \end{align*}
Because $k_c = (\log(1+\eta C))^{-1}$, $(1+\eta C)^k\leq e<3$, we then obtain~\eqref{eq:finite-time stability empirical OGD}.
 
\emph{Statement (ii).} Choose $\lambda_c=2C$ and $\eta_c=(2C)^{-1}$. Then $1-\eta(\lambda-C)\in(0,1)$, and iterating~\eqref{eq:ineql ogd error} gives, for every $k\in\NN$,
\[
   \Delta_k\leq\eta C\sum_{l=0}^{k-1}(1-\eta C)^{k-1-l} \bW_1(\mu_0^l,\widehat\mu_0^l) .
\]
Take expectations and apply Lemma~\ref{lem:wss-1 lln}, we similarly find that
\begin{align*}
        \EE  \Delta_k \leq \eta C n^{-1/d} \sum_{l=0}^{k-1}(1-\eta C)^{k-1-l} \leq (1-(1-\eta C)^{k} )n^{-1/d} \leq n^{-1/d}.
    \end{align*}
\end{proof}
 
\begin{theorem}[Concentration of the empirical OGD]\label{thm:concentration empirical OGD}
Under the assumptions and notations of Theorem~\ref{thm:stability empirical OGD}, there exists a constant $\alpha = \alpha(\bA)$ such that:
\begin{enumerate}[label=\emph{(\roman*)}]
    \item For $\eta>0$, $\lambda\in[0,\eta^{-1})$, we let $k_c \coloneq (\log(1+C\eta))^{-1}$. Then, for every $\varepsilon>0$,
    \begin{equation}\label{eq:concentration empirical OGD finite time}
        \PP\bigl(\bigl|\,\sup_{k \leq k_c}\|\theta_k-\widehat\theta_k\|_{L^\infty}-\EE\sup_{k \leq k_c} \|\theta_k-\widehat\theta_k\|_{L^\infty}\bigr|>\varepsilon\bigr)\leq 2\exp(-\eta^{-1}\alpha n\varepsilon^2).
    \end{equation}
    \item Let $\eta<\eta_c$, $\lambda\in(\lambda_c,\eta^{-1})$. Then, for every $k\in\NN$ and $\varepsilon>0$,
    \begin{equation}\label{eq:concentration empirical OGD long time}
        \PP\bigl(\bigl|\,\|\theta_k-\widehat\theta_k\|_{L^\infty}-\EE\|\theta_k-\widehat\theta_k\|_{L^\infty}\bigr|>\varepsilon\bigr)\leq 2\exp(-\eta^{-1}\alpha n\varepsilon^2).
    \end{equation}
\end{enumerate}
\end{theorem}
 
\begin{proof}
Fix $k\in\NN$ and view $f\coloneqq\|\theta_k-\widehat\theta_k\|_{L^\infty}$ as a function of the $N=kn$ i.i.d.\ samples $\{\{z^{(\ell,i)}\}_{i=1}^n\}_{\ell=0}^{k-1}$. The randomness enters only through $\widehat\theta_k$.
 
\emph{Bounded differences.} Replacing one sample $z^{(\ell,i)}$ by an independent copy $\overline z^{(\ell,i)}$ produces a new initial measure $\overline\mu_0^\ell$ and new empirical OGD trajectories $\{\overline\theta_l\}_{l > \ell}$. We bound 
    \begin{align*}
        \left|\sup_{k\leq k_c}\|\theta_k-\widehat\theta_k\|_{L^\infty}-\sup_{k\leq k_c}\|\theta_k-\overline\theta_k\|_{L^\infty} \right|\leq\sup_{k\leq k_c}\|\widehat\theta_k-\overline\theta_k\|_{L^\infty}.
    \end{align*}
For $l=\ell$, since the perturbed trajectories agree before step $\ell$, applying Theorem~\ref{thm:stability of loss gradient} to the single OGD step gives
\[
    \|\widehat\theta_{\ell+1}-\overline\theta_{\ell+1}\|_{L^\infty}\leq\eta C\,\bW_1(\widehat\mu_0^\ell,\overline\mu_0^\ell)\leq\eta C\cdot\tfrac{2R}{n}.
\]
For $l>\ell$, both trajectories see the same data and Theorem~\ref{thm:stability of loss gradient} yields the contraction
\[
    \|\widehat\theta_{l+1}-\overline\theta_{l+1}\|_{L^\infty}\leq(1-\eta(\lambda-C))\|\widehat\theta_l-\overline\theta_l\|_{L^\infty}.
\]

\emph{Statement (i).} 
Iterating gives the McDiarmid coefficient
\begin{equation*}
    c_i^\ell \coloneqq \sup_{k \leq k_c} \|\widehat\theta_{k}-\overline\theta_{k}\|_{L^\infty} \leq\;(1+\eta C)^{k_c-1-\ell}\,\tfrac{2RC\eta}{n}.
\end{equation*}

Hence,
    \begin{equation}\label{eq:McDiarmid sum 1}
        \sum_{\ell=0} ^{k_c-1} \sum_{i = 1} ^n \left( c_i ^{\ell} \right)^2 \leq \frac{4R^2C^2}{n} \eta^ 2\sum_{\ell=0} ^{k_c-1} (1+\eta C)^{2\ell}.
    \end{equation}

For $\lambda\in[0,\eta^{-1})$ and $k_c = (\log(1+\eta C))^{-1}$, the right-hand side of \eqref{eq:McDiarmid sum 1} is bounded by
\begin{equation}\label{eq:theta_k McDiarmid finite time}
     \frac{4R^2}{n} C\eta (1+\eta C)^{2k_c} \leq \frac{36R^2}{n} C\eta.
\end{equation}
McDiarmid's inequality (Lemma~\ref{lem:McDiarmid ineql}) yields~\eqref{eq:concentration empirical OGD finite time}.
 
\emph{Statement (ii).} 
Iterating gives the McDiarmid coefficient
\begin{equation*}
    c_i^\ell \coloneqq \|\widehat\theta_{k}-\overline\theta_{k}\|_{L^\infty} \leq\;(1-\eta(\lambda-C))^{k-1-\ell}\,\tfrac{2RC\eta}{n}.
\end{equation*}

Hence,
    \begin{equation}\label{eq:McDiarmid sum 2}
        \sum_{\ell=0} ^{k-1} \sum_{i = 1} ^n \left( c_i ^{\ell} \right)^2 \leq \frac{4R^2C^2}{n} \eta^ 2\sum_{\ell=0} ^{k-1} (1-\eta(\lambda-C))^{2\ell}.
    \end{equation}

For $\eta<\eta_c$ and $\lambda>\lambda_c=2C$, the geometric series in~\eqref{eq:McDiarmid sum 2} is bounded by
\begin{equation}\label{eq:theta_k McDiarmid long time}
    \frac{4R^2C^2}{n} \eta^ 2 \frac{1}{1-(1-\eta(\lambda-C))^2} = \frac{4R^2C^2}{n} \eta^ 2 \frac{1}{(2-\eta(\lambda-C))\eta(\lambda-C)}\leq \frac{4R^2C\eta}{n}.
\end{equation}
McDiarmid's inequality yields~\eqref{eq:concentration empirical OGD long time}.
\end{proof}
 
\begin{corollary}[High-probability bounds]\label{cor:concentration empirical OGD high prob}
Under the assumptions of Theorem~\ref{thm:concentration empirical OGD}, for every $\beta>0$:
\begin{enumerate}[label=\emph{(\roman*)}]
    \item In the finite-horizon regime, with probability at least $1-2e^{-\beta}$,
    \[
        \sup_{k \leq k_c}\|\theta_k-\widehat\theta_k\|_{L^\infty}\leq 3n^{-1/d}+C_1 \eta^{1/2} \beta^{1/2}n^{-1/2}.
    \]
    \item In the uniform-in-time regime, with probability at least $1-2e^{-\beta}$,
    \[
        \|\theta_k-\widehat\theta_k\|_{L^\infty}\leq n^{-1/d}+C_1\eta^{1/2}\beta^{1/2}n^{-1/2}.
    \]
\end{enumerate}
\end{corollary}
\begin{proof}
    The proof is a direct combination of \Cref{thm:stability empirical OGD} and \Cref{thm:concentration empirical OGD}. 
\end{proof}
 
\subsection{Proof of Theorem~\ref{thm:OGD POC main}}
 
Choose $\beta=\eta^{-1} n^{1-2/d}$ in Corollary~\ref{cor:concentration empirical OGD high prob}. In statement (i), $C_1\eta^{1/2}\beta^{1/2}n^{-1/2}=C_1n^{-1/d}$, so $\|\theta_k-\widehat\theta_k\|_{L^\infty}\leq Cn^{-1/d}$ with probability at least $1-2\exp(-\eta^{-1} n^{1-2/d})$, which is~\eqref{eq:finite OGD POC}. In statement (ii), $C_1\eta^{1/2}\beta^{1/2}n^{-1/2}= Cn^{-1/d}$, and the probability bound becomes $1-2\exp(-\eta^{-1}n^{1-2/d})$, which is~\eqref{eq:uniform OGD POC}. 

For \eqref{eq:intro finite OGD POC sharp} and \eqref{eq:intro uniform OGD POC sharp} in \Cref{thm:OGD POC main}, they follow from \Cref{thm:OGD POC sharp} if we choose $\beta$ such that $2e^{-\beta} = \delta$.
\qed

\section{Sharp Concentration Rate Under Assumption~\ref{ass:kernel form V}}\label{sec:parametric}
In this section, we prove the sharp $n^{-1/2}$ rate in Theorem~\ref{thm:forward POC main}~and~\ref{thm:OGD POC main}. For any $(x_0, \mu_0) \in \RR^d \times \cP(\RR^d)$, we use $(x_s,\mu_s)$ to denote the solution to the contextual flow ~\eqref{eq:model intro} with initial condition $(x_0,\mu_0)$. The solutions to the adjoint system in \Cref{thm:gradient adjoint} are denoted as $(p_s,\phi_s)$, representing the token adjoint and the measure adjoint respectively. For the random empirical measure $\widehat \mu_0$ sampled from $\mu_0$, i.e., 
    \begin{align*}
        \widehat\mu_0= 
    \frac{1}{n}\sum_{i=1}^n \delta_{z^{(i)}},
    \qquad z^{(i)}\stackrel{\mathrm{iid}}{\sim}\mu_0,
    \end{align*}
we use $(\widehat x_s, \widehat \mu_s)$ to denote the solution to the contextual flow ~\eqref{eq:model intro} with initial condition $( x_0,\widehat \mu_0)$ and the parameter path $\widehat \theta(s)$. Similarly, we use $(\widehat p_s ,\widehat \phi_s)$ to denote the token adjoint and the measure adjoint respectively for the adjoint system in \Cref{thm:gradient adjoint} with $(x_0,\widehat \mu_0,\widehat \theta(s))$. We will finally choose $\widehat \theta(s)$ as those $\widehat\theta_k$ in \Cref{thm:OGD POC main}. In this section, we only need to keep in mind that $\widehat \theta(s)$ is a random parameter path depending on $\widehat \mu_0$ implicitly.

Throughout this section, we use \Cref{ass:sublinearity}, which can be implied directly from \Cref{ass:kernel form V}.
\begin{assumption}\label{ass:sublinearity}
    Let 
        \begin{align}
            \mathscr A
    \coloneqq
    \left\{\cV, \,\tfrac{\delta\cV}{\delta\mu}, \,
        D_x\cV,\,
        D_\theta\cV,\,
        \gradW_\mu\cV
    \right\}
        \end{align}
    denote the functional collections in \Cref{ass:regularity}.
    For any $x,z \in \RR^d,\mu,\nu \in \cP(\RR^d),\theta \in \RR^d$ satisfying \Cref{ass:compact}, and for any $A \in \mathscr A$, we assume that 
        \begin{align}\label{eq:sublinear}
            \left\| A(x,z,\mu;\theta) - A(x,z,\nu;\theta)\right\| \leq M_4 \sum_{h =1} ^H \left| \int_{\RR^d} f_h(x,z,y;\theta) \ \d (\mu-\nu)(y) \right|,
        \end{align}
    for some constant $M_4 >0$ and $H \in \ZZ_+$. Also, there is an $L_4>0$, such that for each $h \in \llbracket 1,H \rrbracket$, $f_h(x,y,z;\theta)$ is a Lipschitz function or Lipschitz map, and for any $x,z,y , \widetilde x, \widetilde z,\widetilde y \in \RR^d$ and any $\theta,\widetilde \theta \in \RR^p$, we have that
    \begin{align}\label{eq:sublinear bound}
            \left|f_h(x,z,y;\theta)  \right| \leq L_4 \left(|x| + |z|+|y| + |\theta| + 1\right),
        \end{align}
    and
        \begin{align}\label{eq:sublinear Lipschitz}
            \left|f_h(x,z,y;\theta) - f_h(\widetilde x, \widetilde z,  \widetilde y; \widetilde \theta) \right| \leq L_4 \left(|x-\widetilde x| +|z-\widetilde z|+ |y-\widetilde y| + |\theta - \widetilde \theta| \right).
        \end{align}
    Here, for $A \in \mathscr A$ without the $z$ component, i.e., $\cV, \,
        D_x\cV,\,
        D_\theta\cV$,  we omit those $z$ components in the above \eqref{eq:sublinear bound} and \eqref{eq:sublinear Lipschitz}.
\end{assumption}

\begin{theorem}\label{thm:forward POC sharp}
    Adopt the settings and assumptions in \Cref{thm:forward POC main} and additionally adopt \Cref{ass:sublinearity}. There exists
a constant $C=C(\bA,M_4,L_4,H)$ such that 
\begin{align}\label{eq:average concentration empirical contextual flow sharp}
            \EE\Bigl[\sup_{s\in[0,1]}|x_s-\widehat x_s|\Bigr]\;&\leq\; C(n^{-1/2}+\EE \|\theta - \widehat \theta\|_{L^1}).
        \end{align}
What's more, if $\widehat \theta(s) =  \theta(s)$, then for every $\beta>0$, with probability at least $1-2e^{-\beta}$,
\begin{equation}\label{eq:forward POC sharp bound}
    \sup_{s\in[0,1]}\, \bigl|x_s-\widehat x_s\bigr|
    \;\leq\; C\,(1+\beta^{1/2}) n^{-1/2}.
\end{equation}
\end{theorem}
\begin{proof}
    Similar to the proof of \Cref{thm: concentration of empirical contextual flow}, \eqref{eq:forward POC sharp bound} follows directly once we obtain \eqref{eq:average concentration empirical contextual flow sharp}. We then only show \eqref{eq:average concentration empirical contextual flow sharp}.

    Recall that
    \begin{subequations}
\begin{empheq}[left=\empheqlbrace]{align*}
&\dot{x}_s = \cV(x_s,\mu_s;\theta(s)), \quad \dot{\widehat x}_s = \cV(\widehat x_s,\widehat \mu_s;\widehat \theta(s)), \\
&\partial_s  \mu_s + \nabla\cdot \bigl(\mu_s\,\cV(\cdot, \mu_s; \theta(s))\bigr) = 0 ,\quad \partial_s \widehat \mu_s + \nabla\cdot \bigl(\widehat \mu_s\,\cV(\cdot,\widehat \mu_s;\widehat \theta(s))\bigr) = 0, \\
&x_s|_{s=0} = x_0, \quad \widehat x_s|_{s=0} = x_0,\\
&\mu_s|_{s=0} = \mu_0, \quad \widehat \mu_s|_{s=0} = \frac{1}{n}\sum_{i=1}^n \delta_{z^{(i)}}. 
\end{empheq}
\end{subequations}
We consider an auxiliary flow solving the system
    \begin{subequations}
\begin{empheq}[left=\empheqlbrace]{align*}
& \dot{\widetilde x}_s = \cV(\widetilde x_s,\widetilde \mu_s; \theta(s)), \\
&\partial_s \widetilde \mu_s + \nabla\cdot \bigl(\widetilde \mu_s\,\cV(\cdot,\widetilde \mu_s; \theta(s))\bigr) = 0 , \\
&\widetilde x_s|_{s=0} = x_0,\\
&\widetilde \mu_s|_{s=0} = \frac{1}{n}\sum_{i=1}^n \delta_{z^{(i)}}. 
\end{empheq}
\end{subequations}
By \Cref{thm:stability flow main} applied to $(\widetilde \mu_s , \widehat \mu_s)$ and $(\widetilde x_s, \widehat x_s)$, we take supremum on \eqref{eq:stab mu main} and \eqref{eq:stab x main} and then take expectations, we have that for some $C=C(\bA)$,
\begin{align*}
    \EE \sup_{s \in [0,1]} \bW_1(\widetilde \mu_s , \widehat \mu_s) \leq C \EE \|\theta - \widehat \theta\|_{L^1} \quad \text{and} \quad \EE\Bigl[\sup_{s\in[0,1]}|\widetilde x_s-\widehat x_s|\Bigr]\leq\; C\EE \|\theta - \widehat \theta\|_{L^1}.
\end{align*}
Hence, we only need to show that for some $C=C(\bA,M_4,L_4,H)$ ,
\begin{align*}
    \EE\Bigl[\sup_{s\in[0,1]}|x_s-\widetilde x_s|\Bigr]\;&\leq\; Cn^{-1/2}. 
\end{align*}
    
    Consider the characteristic ODEs associated with $(x_s,\mu_s)$ and $(\widetilde x_s, \widetilde \mu_s)$: let $\Phi_s,\widetilde\Phi_s:\RR^d\to\RR^d$ solve
\begin{equation}\label{eq:char flow sharp}
    \begin{cases}
    \partial_s\Phi_s(x)=\cV(\Phi_s(x),\mu_s;\theta(s)),\\
    \Phi_0(x)=x,
    \end{cases}
    \qquad
    \begin{cases}
    \partial_s\widetilde\Phi_s(x)=\cV(\widetilde\Phi_s(x),\widetilde\mu_s;\theta(s)),\\
    \widetilde\Phi_0(x)=x.
    \end{cases}
\end{equation}
Then $x_s=\Phi_s(x_0)$, $\widetilde x_s = \widetilde \Phi_s(x_0)$, and $\mu_s=(\Phi_s)_\#(\mu_0)$, $\widetilde \mu_s=(\widetilde \Phi_s)_\#(\widetilde \mu_0)$. Also, $\widetilde \mu_s$ actually has the form
    \begin{align*}
        \widetilde\mu_s = \frac{1}{n}\sum_{i=1}^n\delta_{\widetilde \Phi_s(z^{(i)})}.
    \end{align*}
We define the auxiliary empirical measure as
    \begin{align}\label{eq:auxiliary empirical measure}
        \overline \mu_s \coloneqq  \frac{1}{n}\sum_{i=1}^n\delta_{\Phi_s(z^{(i)})}.
    \end{align}
Then, by \Cref{ass:regularity},
    \begin{align}\label{eq:interpolation dynamic x_s}
        \begin{split}
            \frac{\d}{\d s}|x_s-\widetilde x_s|&\leq |\cV(x_s,\mu_s;\theta(s))- \cV(\widetilde x_s,\widetilde\mu_s;\theta(s))|\\
    &\leq |\cV(x_s,\mu_s;\theta(s))- \cV(x_s,\widetilde\mu_s;\theta(s))| + |\cV(x_s,\widetilde\mu_s;\theta(s))- \cV(\widetilde x_s,\widetilde\mu_s;\theta(s))| \\
    &\leq |\cV(x_s,\mu_s;\theta(s))- \cV(x_s,\widetilde \mu_s;\theta(s))| + L |x_s-\widetilde x_s|
    \\  &\leq \underbrace{|\cV(x_s,\mu_s;\theta(s))- \cV(x_s,\overline \mu_s;\theta(s))|}_{:=A_s} + \underbrace{|\cV(x_s,\overline \mu_s;\theta(s))- \cV(x_s,\widetilde\mu_s;\theta(s))|}_{:=B_s}
    \\  &\quad + L |x_s-\widetilde x_s|.
        \end{split}
    \end{align}
A direct Gr\"onwall argument shows that 
    \begin{align*}
        \sup_{s \in [0,1]} |x_s-\widetilde x_s| \leq e^L \int_0 ^1 A_s +B_s \ \d s,
    \end{align*}
and thus
    \begin{align*}
        \EE\Bigl[\sup_{s\in[0,1]}|x_s-\widetilde x_s|\Bigr] \leq e^L \int_0 ^1 \EE A_s + \EE B_s \ \d s
    \end{align*}
We then estimate $\EE A_s$ and $\EE B_s$ in respectively. For $\EE A_s$, we notice that, by \Cref{ass:sublinearity},
    \begin{align*}
            \begin{split}
                &A_s = \left| \cV(x_s,\mu_s;\theta(s))- \cV(x_s,\overline \mu_s;\theta(s))\right| \leq M_4 \sum_{h =1} ^H \left| \int_{\RR^d} f_h(x_s,y;\theta(s)) \ \d (\mu_s-\overline \mu_s)(y) \right|
                \\  & =  M_4 \sum_{h =1} ^H \left| \int_{\RR^d} f_h(x_s,y;\theta(s)) \ \d (\Phi_s)_\#(\mu_0-\overline \mu_0)(y) \right|
                \\  & =  M_4 \sum_{h =1} ^H \left| \int_{\RR^d} f_h(x_s,\Phi_s(y);\theta(s)) \ \d (\mu_0-\overline \mu_0)(y) \right|.
            \end{split}
        \end{align*}
    For each $h \in \llbracket 1, H \rrbracket$, we define random variables 
        \begin{align*}
            \xi_{h,j} \coloneqq f_h(x_s,\Phi_s(z^{(j)});\theta(s)), \quad j = 1,2,\dots,n,
        \end{align*}
    and $\overline \xi_h \coloneqq \frac{1}{n} \sum_{j=1} ^n \xi_{h,j}$.
    We know that $z^{(1)},\ldots,z^{(n)}\stackrel{\mathrm{iid}}{\sim}\mu_0$. Thus,
        \begin{align}\label{eq:sharp concentration variance}
            \EE \left| \int_{\RR^d} f_h(x_s,\Phi_s(y);\theta(s)) \ \d (\mu_0-\overline \mu_0)(y) \right| = \EE \left| \EE (\overline \xi_h ) - \overline \xi_h \right| \leq \sqrt{\text{Var} \left( \overline \xi_h  \right)} = n^{-1/2} \sqrt{\text{Var} \left(  \xi_{h,1}  \right)},
        \end{align}
    and we have that $\text{Var} \left(  \xi_{h,1}  \right) \leq C$ for some $C= C(\bA,L_4)$ independent of $d$ according to \Cref{ass:sublinearity}, \Cref{ass:compact}, and \Cref{thm:stability of contextual flow}. Thus, for some $C = C(\bA,M_4,L_4,H)$, $\EE A_s \leq C n^{-1/2}$.

    Then, we estimate $\EE B_s$. By \Cref{ass:sublinearity},
        \begin{align*}
            \begin{split}
                &B_s = \left| \cV(x_s,\overline \mu_s;\theta(s))- \cV(x_s,\widetilde \mu_s;\theta(s))\right| \leq M_4 \sum_{h =1} ^H \left| \int_{\RR^d} f_h(x_s,y;\theta(s)) \ \d (\overline \mu_s-\widetilde \mu_s)(y) \right|
                \\  &= \frac{M_4}{n} \sum_{h =1} ^H \left| \sum_{i=1} ^n \left( f_h(x_s, \Phi_s(z^{(i)}) ;\theta(s)) - f_h(x_s, \widetilde \Phi_s(z^{(i)}) ;\theta(s)) \right) \right|
                \\  &\leq \frac{M_4 H L_4}{n}  \sum_{i=1} ^n \left|\Phi_s(z^{(i)}) - \widetilde \Phi_s(z^{(i)})  \right|.
            \end{split}
        \end{align*}
    Denote $W_s = \frac{1}{n} \sum_{i=1} ^n \left|\Phi_s(z^{(i)}) - \widetilde \Phi_s(z^{(i)})  \right|$. Clearly, we have that 
        \begin{align*}
            \bW_1 \left(\overline \mu_s, \widetilde\mu_s \right) \leq W_s.
        \end{align*}
    Then, similar to \eqref{eq:interpolation dynamic x_s}, we have that, for each $i \in \llbracket 1, n \rrbracket$,
        \begin{align*}
            &\frac{\d}{\d s}  \left|\Phi_s(z^{(i)}) - \widetilde \Phi_s(z^{(i)})  \right| \leq \left|\cV(\Phi_s(z^{(i)}),\mu_s;\theta(s))- \cV(\widetilde \Phi_s(z^{(i)}),\widetilde\mu_s;\theta(s)) \right|
            \\  &\leq \left|\cV(\Phi_s(z^{(i)}),\mu_s;\theta(s))- \cV( \Phi_s(z^{(i)}),\overline\mu_s;\theta(s)) \right| + \left|\cV(\Phi_s(z^{(i)}),\overline \mu_s;\theta(s))- \cV(\widetilde \Phi_s(z^{(i)}),\widetilde\mu_s;\theta(s)) \right|
            \\  &\leq \left|\cV(\Phi_s(z^{(i)}),\mu_s;\theta(s))- \cV( \Phi_s(z^{(i)}),\overline \mu_s;\theta(s)) \right| + L_0 \left( \left|\Phi_s(z^{(i)}) - \widetilde \Phi_s(z^{(i)})  \right| +\bW_1 \left(\overline \mu_s, \widetilde\mu_s \right) \right)
            \\  &\leq \left|\cV(\Phi_s(z^{(i)}),\mu_s;\theta(s))- \cV( \Phi_s(z^{(i)}),\overline \mu_s;\theta(s)) \right| + L_0 \left( \left|\Phi_s(z^{(i)}) - \widetilde \Phi_s(z^{(i)})  \right| +W_s \right)
        \end{align*}
    where we applied \eqref{ass:V}. Thus,
        \begin{align*}
            \sup_{s \in [0,1]} W_s \leq \frac{e^{2L_0}}{n} \sum_{i=1} ^n  \int_0 ^1 \left|\cV(\Phi_s(z^{(i)}),\mu_s;\theta(s))- \cV( \Phi_s(z^{(i)}),\overline \mu_s;\theta(s)) \right| \ \d s
        \end{align*}
    Similar to the estimates for $\EE A_s$, use \Cref{ass:sublinearity}, we have that 
        \begin{align}\label{eq: Ws sharp expectation}
            \EE \sup_{s \in [0,1]} W_s \leq C n^{-1/2},
        \end{align}
    for some $C = C(\bA,M_4,L_4,H)$. The only mild modification is that we need to condition on $z^{(i)}$, and then only the random variables $f_h(\Phi_s(z^{(i)}),\Phi_s(z^{(j)});\theta(s))$, with $j \neq i$ are i.i.d.. But the contribution of the term $f_h(\Phi_s(z^{(i)}),\Phi_s(z^{(i)});\theta(s))$ in \eqref{eq:sharp concentration variance} is $n^{-1}$, so we can still get the $n^{-1/2}$ rate in the expectation. Thus, we have that $\EE B_s \leq M_4 HL_4 \EE W_s \leq C n^{-1/2}$.

    Combine the estimates for $\EE A_s$ and $\EE B_s $, we obtain \eqref{eq:average concentration empirical contextual flow sharp}.
\end{proof}

We remark that one important by-product in the estimates of the term $\EE B_s $ is that
    \begin{align}\label{eq:shadow measure W1 sharp}
        \EE \sup_{s \in [0,1]} \bW_1 \left(\overline \mu_s, \widehat\mu_s \right) \leq \EE \sup_{s \in [0,1]} W_s + \EE \sup_{s \in [0,1]} \bW_1(\widetilde \mu_s , \widehat \mu_s) \leq C(n^{-1/2}+\EE \|\theta - \widehat \theta\|_{L^1}),
    \end{align}
which will be essential in the proof of the following \Cref{thm:adjoint system sharp} for the expectations of the adjoint system in \Cref{thm:gradient adjoint}. That is, \eqref{eq:gradient}, \eqref{eq:adjoint token}, and \eqref{eq:adjoint measure}, also satisfy this sharp $n^{-1/2}$-rate.

\begin{theorem}\label{thm:adjoint system sharp}
    Adopt the settings and assumptions in \Cref{thm:gradient adjoint} and additionally adopt \Cref{ass:kernel form V}. There exists
a constant $C=C(\bA,L_4)$ such that
    \begin{align}\label{eq:token adjoint expectation sharp}
    \EE\left[\sup_{s\in[0,1]}|p_s-\widehat p_s|\right]
    \leq
    C (n^{-1/2} + \EE \|\theta - \widehat \theta\|_{L^1}),
\end{align}
and
    \begin{align}\label{eq:measure adjoint expectation sharp}
    \EE \left[\sup_{s\in[0,1],z \in B_R(0)}|\nabla\phi_s(z)-\nabla\widehat\phi_s(z)| \right]\leq C \left(d^{1/2}n^{-1/2} + \EE \|\theta - \widehat \theta\|_{L^1} \right).
    \end{align}
\end{theorem}
\begin{proof}
    Adopt the $(\widetilde x_s ,\widetilde \mu_s)$ defined in the proof of \Cref{thm:forward POC sharp}. First, by \Cref{prop: bound p}, we can assume that both $p_s,\widehat p_s$ are bounded by a $C=C(\bA)$. For \eqref{eq:token adjoint expectation sharp}, by \eqref{eq:adjoint token} and those Lipschitz assumptions in \Cref{ass:regularity}, we have that
    \begin{align*}
        |p_s-\widehat p_s|
        &\leq |p_1-\widehat p_1|+\int_s^1\Bigl\{
        \bigl|[D_x\cV(x_\tau,\mu_\tau;\theta(\tau))-D_x\cV(x_\tau,\overline\mu_\tau;\theta(\tau))]^\top p_\tau\bigr|\\
        &\qquad+C\bigl(\bW_1(\overline\mu_\tau,\widehat\mu_\tau)
        +|x_\tau-\widehat x_\tau|+|p_\tau-\widehat p_\tau|
        +|\theta(\tau)-\widehat\theta(\tau)|\bigr)\Bigr\}\,\d\tau,
    \end{align*}
    for some $C= C(\bA)$. Then, \eqref{eq:token adjoint expectation sharp} follows from combining \eqref{eq:average concentration empirical contextual flow sharp} and a similar argument for the estimates of the terms $\EE A_s, \EE B_s$ in the proof of \Cref{thm:forward POC sharp}.

    Next, similar to the proof of \Cref{prop: stability phi_s}, we define $\Gamma_s(z)\coloneqq \widehat \phi_s(z)-\phi_s(z)$. Subtracting the two equations~\eqref{eq:adjoint measure} and rearranging,
\[
    \partial_s\Gamma_s(z)+\nabla_z(\Gamma_s(z))\!\cdot\!\cV(z, \widehat \mu_s;\widehat \theta(s))=h(s,z)+g(s,z),
\]
with $\Gamma_1\equiv 0$ and source terms
\begin{align*}
    h(s,z) = -\int_{\RR^d}\!\nabla(\Gamma_s(\xi))\!\cdot\!\tfrac{\delta\cV}{\delta\mu}[\xi,\widehat \mu_s;\widehat \theta(s)](z)\d\widehat \mu_s(\xi),
\end{align*}
and
\begin{align}\label{eq:g(s,z) sharp}
    \begin{split}
        g(s,z)
    &=-\nabla\phi_s(z)\!\cdot\!\bigl(\cV(z,\widehat \mu_s;\widehat \theta(s))-\cV(z,\mu_s;\theta(s))\bigr)\\
    &\quad-\bigl(\widehat p_s\!\cdot\!\tfrac{\delta\cV}{\delta\mu}[\widehat x_s,\widehat \mu_s;\widehat \theta(s)](z)- p_s\!\cdot\!\tfrac{\delta\cV}{\delta\mu}[x_s,\mu_s;\theta(s)](z)\bigr)\\
    &\quad-\!\int_{\RR^d}\!\nabla \phi_s(\xi)\!\cdot\!\bigl(\tfrac{\delta\cV}{\delta\mu}[\xi,\widehat \mu_s;\widehat \theta(s)](z)\,\d \widehat \mu_s(\xi)-\tfrac{\delta\cV}{\delta\mu}[\xi,\mu_s;\theta(s)](z)\,\d\mu_s(\xi)\bigr).
    \end{split}
\end{align}
We used similar notations $\Gamma_s(z)$ and $g(s,z)$ as in the proof of \Cref{prop: stability phi_s}. 

Similar to the proof of \Cref{prop:existence phi_s}, we let $\widehat \Psi_\tau(z)$ be the solution from depth $s$ of $\partial_\tau\widehat \Psi_\tau=\cV(\widehat \Psi_\tau,\widehat \mu_\tau;\widehat \theta(\tau))$ with $\widehat \Psi_s(z)=z$, $\forall z\in \RR^d$, we have that
    \begin{align}\label{eq: Delta contraction map sharp}
       \Gamma_s(z)\;=\;-\int_s^1 h(\tau,\widehat \Psi_\tau(z)) + g(\tau,\widehat \Psi_\tau(z)) \,\d\tau.
    \end{align}
Denote
    \begin{align*}
        \| \Gamma \|_{R} \coloneqq \sup_{s \in [0,1] , z \in B_R(0)} e^{\alpha(s-1)} |\nabla \Gamma_s(z)|.
    \end{align*}
Because $\supp \widehat \mu_s \subseteq B_R(0)$ according to \eqref{ass:compact}, we have that $\|h(s,\widehat \Psi_s(z))\|_R \leq C\|\Gamma \|_{R}$ for some $C = C(\bA)$ using those Lipschitz assumptions on $\cV$ in \Cref{ass:regularity}. Thus, by \eqref{eq: Delta contraction map sharp}
    \begin{align*}
        \begin{split}
            \|\Gamma\|_R &\leq \sup_{s \in [0,1], z \in B_R(0)} e^{\alpha(s-1)} \int_s ^1 e^{\alpha(1-\tau)}C\|\Gamma\|_R + C|\nabla g(\tau,\widehat \Psi_\tau(z))| \,\d\tau
            \\  & \leq \frac{C}{\alpha} \|\Gamma\|_R + \sup_{s \in [0,1], z \in B_R(0)}\int_s ^1 C|\nabla g(\tau,\widehat \Psi_\tau(z))| \,\d\tau .
        \end{split}
    \end{align*}
Thus, if we take $\alpha = 2C$, and take expectations on the above inequality, we get that
    \begin{align*}
        \begin{split}
            \EE \|\Gamma\|_R &\leq 2C \EE  \sup_{s \in [0,1] , z \in B_R(0)}\int_s ^1 |\nabla g(\tau,\widehat \Psi_\tau(z))| \,\d\tau
            \\  &\leq 2C \int_0 ^1 \EE\sup_{z\in B_{R'}(0)}|\nabla g(\tau,z)|\,\d\tau,
        \end{split}
    \end{align*}
where in the second inequality, $R' = R'(\bA)$ obtained in \eqref{eq:support mu} of \Cref{thm:stability of contextual flow} to control the norm of $|\widehat \Psi_\tau(z)|$ for $z \in B_R(0)$. After bounding the Hessian factors by the $C^{1,1}$ estimate, the remaining empirical kernels are Lipschitz in $z$ under Assumptions~\ref{ass:regularity} and~\ref{ass:kernel form V}. Hence, similar to the proof for \Cref{thm:forward POC sharp}, we combine with the Dudley’s entropy integral among the class of Lipschitz functions $\{|\nabla g(\tau,z)|\}_{z \in B_{R'}(0)} $ (see for example \cite[Chapter 2]{van1996weak} and \cite[Chapter 2]{kosorok2008introduction}), we have a $C= C(\bA)$, such that
    \begin{align*}
        \EE  \sup_{ z \in B_{R'}(0)} |\nabla g(\tau,z)| \leq C d^{1/2} n^{-1/2} + \EE |\theta(\tau) - \widehat \theta(\tau)|+C\EE\|\theta-\widehat\theta\|_{L^1}.
    \end{align*}
For example, by the form of $g(\tau,z)$ in \eqref{eq:g(s,z) sharp}, we need to estimate the term
    \begin{align*}
        \begin{split}
            &\EE \sup_{ z \in B_{R'}(0)} \left\| D_x\cV(z,\widehat \mu_\tau;\widehat \theta(\tau))-D_x\cV(z,\mu_\tau;\theta(\tau))\right\| 
            \\  &\leq \EE \sup_{ z \in B_{R'}(0)} \left\| D_x\cV(z,\overline  \mu_\tau;\theta(\tau))-D_x\cV(z,\mu_\tau;\theta(\tau))\right\| + C\left(\EE  \bW_1 \left(\overline \mu_\tau, \widehat\mu_\tau \right)+\EE |\theta(\tau) - \widehat \theta(\tau)| \right) ,
        \end{split}
    \end{align*}
where $C =C(\bA)$. Combine \Cref{ass:sublinearity} for $D_x \cV$ and \eqref{eq:shadow measure W1 sharp}, together with the Dudley's entropy integral, we can bound the above term by $C(\sqrt{d/n}+\EE|\theta(\tau)-\widehat\theta(\tau)|+\EE\|\theta-\widehat\theta\|_{L^1})$. Other terms in $g(\tau,z)$ in \eqref{eq:g(s,z) sharp} can be estimated similarly, combining with \eqref{eq:token adjoint expectation sharp} and \eqref{eq:average concentration empirical contextual flow sharp}. This finishes the proof for \eqref{eq:measure adjoint expectation sharp}.
\end{proof}

\begin{theorem}\label{thm:OGD POC sharp}
    Adopt the settings and assumptions in \Cref{thm:OGD POC main} and additionally adopt \Cref{ass:kernel form V}. There exist
constants $C, \eta_c, \lambda_c$ depending on $\bA,L_4$, such that the following hold.
\begin{enumerate}
    \item[\emph{(i)}] \emph{Small or no ridge penalty.}
    For $\eta>0$, $\lambda\in[0,\eta^{-1})$, we let $k_c \coloneq (\log(1+C\eta))^{-1}$. Then,
    \begin{equation}\label{eq:finite OGD POC sharp}
        \sup_{k \leq k_c}\bigl\|\theta_k-\widehat\theta_k\bigr\|_{L^\infty}
        \;\leq\; C\, (1+\eta^{1/2}\beta^{1/2}) d^{1/2} n^{-1/2},
    \end{equation}
    with probability at least $1-2e^{-\beta}$.

    \item[\emph{(ii)}] \emph{Large ridge penalty.}
    If $\eta<\eta_c$ and $\lambda\in(\lambda_c,\eta^{-1})$, then for any $k\in\NN$,
    \begin{equation}\label{eq:uniform OGD POC sharp}
        \bigl\|\theta_k-\widehat\theta_k\bigr\|_{L^\infty}
        \;\leq\; C\, (1+\eta^{1/2}\beta^{1/2}) d^{1/2} n^{-1/2},
    \end{equation}
    with probability at least $1-2e^{-\beta}$.
\end{enumerate}

\end{theorem}
\begin{proof}
    According to the proof of \Cref{cor:concentration empirical OGD high prob} we only need to modify \Cref{thm:stability empirical OGD} and show that
        \begin{align*}
            \EE\|\theta_k-\widehat\theta_k\|_{L^\infty}\leq 3 d^{1/2}n^{-1/2} ,
        \end{align*}
    in both the small or no ridge penalty case and the large ridge penalty case. We only discuss the small ridge penalty case in the following, and the case for the large ridge penalty is very similar to \Cref{thm:stability empirical OGD}.

    Set $\Delta_l\coloneqq\|\widehat \theta_l-\theta_l\|_{L^\infty}$. Subtracting the two OGD updates \eqref{eq:OGD}, similar to \eqref{eq:ineql ogd error}, we first take $\sup_{s \in [0,1]}$ then take $\sup_{k \leq k_c}$, and finally apply expectations. Combine \Cref{thm:forward POC sharp}, \eqref{eq:shadow measure W1 sharp}, \Cref{thm:adjoint system sharp}, and the arguments in \Cref{thm:stability empirical OGD} to expand those error terms, we can obtain that
\begin{equation}\label{eq:ineql ogd error sharp}
        \EE  \sup_{k \leq k_c} \Delta_k \leq \eta C d^{1/2}n^{-1/2} \sum_{l=0}^{k_c-1}(1+\eta C)^{k_c-1-l} \leq (1+\eta C)^{k_c}  d^{1/2}n^{-1/2} .
\end{equation}
We remark some details when subtracting the two OGD updates \eqref{eq:OGD} for each $l = 0,1,\dots,k_c$. Similar to the arguments in \Cref{thm:stability empirical OGD}, the main error term comes from
    \begin{align*}
        \begin{split}
            &\EE \left\|\frac{\delta\cL(\widehat \theta_l;\widehat \mu_0 ^l,x_0 ^l,y_0 ^l)}{\delta\theta}(s) - \frac{\delta\cL(\theta_l;\mu_0 ^l,x_0 ^l,y_0 ^l)}{\delta\theta}(s) \right\|_{L^{\infty}} 
            \\  &\leq \underbrace{\EE \left\|D_\theta\cV(\widehat x_s ^l,\widehat  \mu_s ^l;\widehat  \theta_l(s))^{\top} \widehat p_s ^l - D_\theta\cV(x_s ^l,\mu_s ^l;\theta_l(s))^{\top} p_s ^l \right\|_{L^{\infty}} }_{\text{Term 1}}
            \\  &\quad + \underbrace{\EE \left\|\int_{\RR^d}D_\theta\cV(z,\widehat \mu_s ^l;\widehat \theta_l (s))^{\top}\,\nabla \widehat \phi_s ^l(z)\,\d \widehat \mu_s ^l(z) - D_\theta\cV(z,\mu_s ^l;\theta_l (s))^{\top}\,\nabla\phi_s ^l(z)\,\d\mu_s ^l(z) \right\|_{L^{\infty}} }_{\text{Term 2}} .
        \end{split}
    \end{align*}
The Term 1 on the right hand side can be estimated by combining \eqref{eq:average concentration empirical contextual flow sharp}, \eqref{eq:shadow measure W1 sharp}, and \Cref{thm:adjoint system sharp}. For the Term 2, we can decompose it by
    \begin{align*}
        \begin{split}
            &\EE \left\|\int_{\RR^d}D_\theta\cV(z,\widehat \mu_s ^l;\widehat \theta_l (s))^{\top}\,\nabla \widehat \phi_s ^l(z)\,\d \widehat \mu_s ^l(z) - D_\theta\cV(z,\mu_s ^l;\theta_l (s))^{\top}\,\nabla\phi_s ^l(z)\,\d\mu_s ^l(z) \right\|_{L^{\infty}}  
            \\ &\leq C \Bigg(\EE \Delta_l +\EE \sup_{s \in[0,1],z \in B_R(0)} \left\|D_\theta\cV(z,\overline \mu_s ^l; \theta_l (s)) -  D_\theta\cV(z, \mu_s ^l; \theta_l (s)) \right\| 
            \\  &\quad +\EE \left[\sup_{s\in[0,1],z \in B_R(0)}|\nabla\phi_s ^l(z)-\nabla\widehat\phi_s ^l(z)| \right] 
            \\  &\quad +  \EE \sup_{s \in [0,1]} \left|\int_{\RR^d}D_\theta\cV(z,\mu_s ^l;\theta_l (s))^{\top}\,\nabla\phi_s ^l(z)\,\d(\overline\mu_s ^l - \mu_s ^l)(z) \right| + \EE\sup_{s \in [0,1]} \bW_1(\overline\mu_s ^l, \widehat \mu_s ^l)\Bigg),
        \end{split}
    \end{align*}
for some $C=C(\bA)$.
Similar to the proofs of \Cref{thm:stability of loss gradient,thm:stability OGD}, the Lipschitz bound in \eqref{ass:compact} propagates uniformly along OGD in the respective regimes. The Dudley argument in the proof of \Cref{thm:adjoint system sharp} therefore applies jointly in depth and space, and together with \eqref{eq:average concentration empirical contextual flow sharp} and \eqref{eq:shadow measure W1 sharp}, bounds Term~2 by $C(\EE\Delta_l+d^{1/2}n^{-1/2})$ for some $C=C(\bA,L_4)$.

After obtaining \eqref{eq:ineql ogd error sharp}, the remaining steps are exactly the same as the proofs in \Cref{thm:stability empirical OGD}.
\end{proof}

\section{Verification of Assumption~\ref{ass:regularity}~and~\ref{ass:kernel form V} for transformer architectures}\label{sec:verify transformer}
 
This appendix verifies that the standard transformer building blocks---self-attention layers~\eqref{eq:tf ic} and MLP layers---satisfy Assumptions~\eqref{ass:V},~\eqref{ass:derivative},~and~\eqref{ass:kernel form V}, with constants that depend only on the support radius $R$ in~\eqref{ass:compact}, the parameter bound $M$, and (for MLPs) the activation regularity. In particular all constants are independent of the embedding dimension $d$ and of the parameter dimension $p$.
 
\subsection{Self-attention blocks}\label{sec:verify attention block}
 
Recall that for a single-head self-attention block, the velocity field~\eqref{eq:tf ic} is
\begin{equation}\label{eq:tf ic proof}
    \cV(x,\mu;\theta)\;=\;\frac{1}{Z(x)}\!\!\int_{\RR^d}\!e^{\langle Qx,Ky\rangle}Vy\,\d\mu(y),
    \qquad
    Z(x)=\!\!\int_{\RR^d}\!e^{\langle Qx,Kz\rangle}\d\mu(z),
\end{equation}
with $\theta=(Q,K,V)\in\RR^{3d^2}$. We first establish uniform boundedness along the contextual flow.
 
\begin{proposition}[Boundedness of the transformer flow]\label{prop:bdd transformer}
Adopt~\eqref{ass:compact} for $\theta$ and $\mu_0$. Let $\mu_s$ solve~\eqref{eq:model intro}. Then for all $s\in[0,1]$,
\[
    \supp\mu_s\subseteq B_{e^M R}(0).
\]
\end{proposition}
 
\begin{proof}
    Let $R_{\tau} \coloneqq \sup \{ |x| \mid x \in \supp \mu_\tau\}$, then $R_0 \leq R$. First, because $\forall s \in [0,1]$, $|\theta(s)|\leq M$, the Frobenius norms $\|Q\|_F,\|K\|_F,\|V\|_F$ are bounded by $M$. Thus $\|Q\|,\|K\|,\|V\|$ are also bounded by $M$. Hence, $\forall x \in \RR^d$, $\forall \tau \in [0,1]$,
        \begin{align*}
            |\cV(x,\mu_{\tau};\theta)| \leq \frac{1}{Z(x)}
\int_{\RR^d} 
\exp\left(\langle Q x, K y \rangle \right) \|V\| |y| \ \d \mu_{\tau}(y) \leq M R_{\tau}
        \end{align*}
        Fix a $\tau >0$, let $\Psi$ be the characteristic flow map of $\cV(z,\mu_s;\theta(s))$ from $\tau$ to $1$, i.e., 
        \begin{align*}
            \partial_s \Psi_s(z) = \cV(\Psi_s(z),\mu_s;\theta(s)), \quad \Psi_{\tau}(z) = z.
        \end{align*}
        Let $z_{\tau} \in \RR^d$ be a point that attains $R_{\tau}$, then
            \begin{align*}
                \partial_s |\Psi_s(z_{\tau})| \bigg|_{s = \tau} \leq  |(\partial_{s} \Psi_{s})(z_{\tau})| \bigg|_{s = \tau}= |\cV(z_{\tau},\mu_\tau;\theta(\tau))| \leq MR_{\tau}.
            \end{align*}
        Because this is true for any $z_{\tau}$ that attains $R_{\tau}$, we must have that
            \begin{align*}
                \partial_s R_s |_{s = \tau} \leq MR_{\tau}.
            \end{align*}
        Hence, by a Gr\"onwall argument, $R_s \leq e^{Ms} R_0$, $\forall s \in [0,1]$.
    
\end{proof}
 
We may therefore verify Assumptions~\eqref{ass:V} and~\eqref{ass:derivative} on $B_R(0)$, treating $R$ as the support radius (after the substitution $R\leftarrow e^M R$ if needed). For notational simplicity we keep $R$ throughout.
 
\paragraph{Notation.} Set
\begin{align*}
    \begin{split}
        & \alpha(x,y) \coloneqq \frac{e^{\langle Qx, Ky\rangle}}{Z(x)},\qquad
    \cU(x)\;\coloneqq\!\int_{\RR^d} \!\alpha(x,y)\,y\,\d\mu(y),\qquad
    \\  & \cM(x)\;\coloneqq\!\int_{\RR^d} \!\alpha(x,y)\,(y-\cU(x))(y-\cU(x))^{\!\top}\,\d\mu(y),
    \end{split}
\end{align*}
so that $\cV(x,\mu;\theta)=V\cU(x)$ and $\|\cM(x)\|\leq 4R^2$ when $x \in B_R(0)$ and $\supp \mu \subseteq B_R(0)$.
 
\begin{proposition}[Lipschitz continuity of self-attention]\label{prop:transformer assumption 1}
Adopt~\eqref{ass:compact}. There exist $M_0=M_0(M,R)$ and $L_0=L_0(M,R)$ such that $\cV$ and $\delta\cV/\delta\mu$ in~\eqref{eq:tf ic proof} satisfy~\eqref{ass:V} for $x,z\in B_R(0)$.
\end{proposition}
 
\begin{proof}
\emph{Lipschitz in $x$.} A direct computation gives
\begin{equation}\label{eq:tf D alpha}
    D_x\alpha(x,y)\;=\;\alpha(x,y)\bigl(y-\cU(x)\bigr)^{\!\top}K^{\!\top}Q,
\end{equation}
so that
\begin{equation}\label{eq:tf D U(x)}
    D_x\cU(x)\;=\;\cM(x)K^{\!\top}Q,
\end{equation}
and consequently
\begin{equation}\label{eq:tf Dx}
    D_x\cV(x,\mu;\theta)\;=\;V\cM(x)K^{\!\top}Q,\qquad\|D_x\cV\|\leq 4M^3R^2.
\end{equation}
 
\emph{Lipschitz in $\mu$.} A direct computation gives
\[
    \tfrac{\delta\cV}{\delta\mu}[x,\mu;\theta](z)\;=\;\alpha(x,z)\,V\bigl(z-\cU(x)\bigr),
\]
and since $\nabla_z\alpha(x,z)=\alpha(x,z)\,K^{\!\top}Qx$,
\begin{equation}\label{eq:tf Dmu}
    \gradW_\mu\cV(z)=\alpha(x,z)\bigl[V(z-\cU(x))(K^{\!\top}Qx)^{\!\top}+V\bigr],
    \qquad
    \|\gradW_\mu\cV\|\leq e^{2M^2R^2}M(2M^2R^2+1).
\end{equation}
 
\emph{Lipschitz in $\theta=(Q,K,V)$.} In the following, we use the notations
    \begin{align*}
        D_V \cV \cdot [\Delta V] \coloneqq \lim_{t \to 0} \frac{\cV(x,\mu;(Q,K,V + t \Delta V)) - \cV(x,\mu;(Q,K,V ))}{t},
    \end{align*} 
where $\Delta V \in \RR^{d \times d}$ is any $d\times d$ matrix. We similarly define $ D_Q\cV \cdot [\Delta  Q]$ and $ D_K \cV \cdot [\Delta  K]$. A direct computation gives
 \begin{align}\label{eq:tf D alpha Q}
    D_Q \alpha(x,y)\cdot[\Delta Q] = \alpha(x,y)\bigl(\langle \Delta Q\,x,Ky\rangle - \langle \Delta Q\,x,K \cU(x)\rangle\bigr) = \alpha(x,y)\,\langle \Delta Q\,x,\,K(y-\cU(x))\rangle.
\end{align}
Thus, the directional derivatives of $\cV$ in $\theta$ are
\begin{align}
    D_V\cV\cdot[\Delta V] &=\Delta V\,\cU(x),\label{eq:tf DV}\\
    D_Q\cV\cdot[\Delta Q] &=V\cM(x)K^{\!\top}\Delta Q\,x,\label{eq:tf DQ}\\
    D_K\cV\cdot[\Delta K] &=V\cM(x)(\Delta K)^{\!\top}Q\,x,\label{eq:tf DK}
\end{align}
with $\|D_V\cV\|\leq R$, $\|D_Q\cV\|,\|D_K\cV\|\leq 4M^2R^3$, hence $\|D_\theta\cV\|\leq R(1+8M^2R^2)$. The linear-growth bound for $\cV$ on $B_R(0)$ follows from boundedness on $B_R(0)$. The same arguments applied to $\delta\cV/\delta\mu$ yield the corresponding bounds.
\end{proof}
 
\begin{proposition}[Smoothness of $D_x\cV$]\label{prop:transformer assumption 2.1}
Adopt~\eqref{ass:compact}. There exists $L_1=L_1(M,R)$ such that $D_x\cV$ in~\eqref{eq:tf ic proof} satisfies~\eqref{ass:derivative}\,(a) for $x\in B_R(0)$.
\end{proposition}
 
\begin{proof}
We use the notation 
    \begin{align*}
        D^2 _x \cV \cdot [h] \coloneqq \lim_{t \to 0} \frac{D_x \cV (x+th,\mu;\theta) - D_x \cV (x,\mu;\theta)}{t},
    \end{align*}
where $h \in \RR^d$ is any $d$-dimensional vector. Thus, the second-order spatial derivative is
\[
    D_x^2\cV\cdot[h]=V\!\!\int\!\alpha(x,y)\bigl[(y-\cU(x))^{\!\top}K^{\!\top}Qh\bigr](y-\cU(x))(y-\cU(x))^{\!\top}\,\d\mu(y)\,K^{\!\top}Q,
\]
so $\|D_x^2\cV\|\leq 8M^5R^3$. 

The variation of $D_x\cV$ in $\mu$ is computed using 
\begin{align}\label{eq:tf D alpha mu}
    \frac{\delta\alpha(x,y)}{\delta \mu}(z) = -\alpha(x,y)\,\alpha(x,z), \quad \frac{\delta\cU(x)}{\delta \mu}(z) = \alpha(x,z) (z-\cU(x)) .
\end{align}
A direct computation then shows that
\begin{align*}
    \frac{\delta(D_x\cV)}{\delta\mu}(z) = V\,\alpha(x,z)\bigl[(z - \cU(x))\,(z - \cU(x))^\top - \cM(x)\bigr]\,K^\top Q,
\end{align*}
and thus
\begin{align*}
    &\gradW_\mu(D_x\cV)(z)\cdot [h] 
    \\  &= V\,\alpha(x,z)\Bigl[\bigl((z - \cU(x))\,(z - \cU(x))^\top - \cM(x)\bigr)\langle K^\top Qx,\,h\rangle 
    \\  & \quad + (z - \cU(x))\,h^\top + h\,(z - \cU(x))^\top\Bigr]K^\top Q.
\end{align*}
Hence, $\|\gradW_\mu(D_x\cV)(z)\|\leq 4M^3Re^{2M^2R^2}(1+2M^2R^2)$. 

The variation in $\theta$ is computed similarly. By the fact that $\int_{\RR^d} \alpha(x,y)\,(y - \cU(x))\,\d\mu(y) = \cU(x) - \cU(x) = 0$, a direct computation gives that
\begin{align*}
    D_V(D_x\cV)\cdot[\Delta V] = \Delta V\,\cM(x)\,K^\top Q,
\end{align*}
and
\begin{align}\label{eq:tf D V x Q}
    \begin{split}
        D_Q(D_x\cV)\cdot[\Delta Q] &= V\left[\int\alpha(x,y)\,\langle \Delta Q\,x,\,K(y-\cU(x)) \rangle\;(y-\cU(x))\,(y-\cU(x))^\top\;d\mu(y)\right]K^\top Q 
        \\  & \quad +V\,\cM(x)\,K^\top \Delta Q,
    \end{split}
\end{align}
and, similarly, 
\begin{align*}
    \begin{split}
        D_K(D_x\cV)\cdot[\Delta K] &= V\left[\int\alpha(x,y)\,\langle Qx,\,\Delta K(y-\cU(x))\rangle\;(y-\cU(x))\,(y-\cU(x))^\top\;d\mu(y)\right]K^\top Q 
        \\ & \quad +V\,\cM(x)\,(\Delta K)^\top Q.
    \end{split}
\end{align*}
Collecting terms gives $\|D_\theta(D_x\cV)\|\leq 12M^2R^2(1+2M^2R^2)$.
\end{proof}

\begin{proposition}[Smoothness of $\gradW_\mu\cV$]\label{prop:transformer assumption 2.2}
Adopt~\eqref{ass:compact}. There exists $L_2=L_2(M,R)$ such that $\gradW_\mu\cV$ in~\eqref{eq:tf ic proof} satisfies~\eqref{ass:derivative}\,(b) for $x,z\in B_R(0)$.
\end{proposition}
 
\begin{proof}
We differentiate $\gradW_\mu\cV(z)=\alpha(x,z)[V(z-\cU(x))(K^{\!\top}Qx)^{\!\top}+V]$ with respect to $x$, $z$, $\mu$, and $\theta$. 

According to \eqref{eq:tf D alpha} and \eqref{eq:tf D U(x)}, a direct computation shows that
\begin{align*}
    \begin{split}
        D_x(\gradW_\mu\mathcal{V})(z)\cdot [h] &= \alpha(x,z)\Big[\langle Q^\top K (z - \cU(x)) ,\,h\rangle\bigl(V(z - \cU(x)) x^{\top} Q^{\top} K + V\bigr) 
        \\  & \quad - V\cM(x)K^\top Qh\; x^{\top} Q^{\top} K + V(z - \cU(x))\,h^\top Q^\top K\Big],
    \end{split}
\end{align*}
and
\begin{align*}
    D_z(\gradW_\mu\cV)(z)\cdot [h] = \alpha(x,z)\Big[x^{\top} Q^{\top} K h\bigl(V(z - \cU(x)) x^{\top} Q^{\top} K + V\bigr) + Vh \, x^{\top} Q^{\top} K\Big].
\end{align*}
Thus, 
\begin{align*}
    \|D_x(\gradW_\mu\mathcal{V})(z)\| \leq e^{2M^2 R^2} 4M^3R(1+2M^2R^2),
\end{align*}
and
\begin{align*}
\|D_z(\gradW_\mu\cV)(z)\| \leq e^{2M^2 R^2} 2M^3R(1+ M^2R^2).
\end{align*}

For the derivatives with respect to $\mu$, we have that, by \eqref{eq:tf D alpha} and \eqref{eq:tf D alpha mu},
\begin{align*}
    \frac{\delta(\gradW_\mu\cV)(z)}{\delta\mu}(w) = -\alpha(x,z)\,\alpha(x,w)\,\bigl[V(z - \cU(x)) x^{\top} Q^{\top} K+ V + V(w-\cU(x))\,x^{\top} Q^{\top} K \bigr]
\end{align*}
and
\begin{align*}
    \begin{split}
        &\gradW_\mu((\gradW_\mu\cV)(z))(w)\cdot [h] 
        \\ &= -\alpha(x,z)\,\alpha(x,w)\Big[\langle Q^{\top} K (w - \cU(x)),h\rangle\bigl[V(z - \cU(x)) x^{\top} Q^{\top} K+ V 
        \\  & \quad + V(w-\cU(x))\,x^{\top} Q^{\top} K\bigr] + Vh\,x^{\top} Q^{\top} K \Big].
    \end{split}
\end{align*}
Thus, we get
\begin{align*}
    \|\gradW_\mu((\gradW_\mu\cV)(z))(w)\| \leq e^{4M^2 R^2} 4 M^3R(1+M^2R^2).
\end{align*}

Finally, we prove the Lipschitz continuity of $\gradW_\mu \mathcal{V}$ in $\theta=(Q,K,V)$. A direct computation shows that
\begin{align*}
    D_V(\gradW_\mu\cV)(z)\cdot[\Delta V] = \alpha(x,z)\Big[(\Delta V)(z - \cU(x)) x^{\top} Q^{\top} K + (\Delta V)\Big].
\end{align*}
By \eqref{eq:tf D alpha Q},
\begin{align*}
    \begin{split}
        D_Q(\gradW_\mu\cV)(z)\cdot[\Delta Q] &= \alpha(x,z)\Big[\langle \Delta Q\,x,K(z - \cU(x))\rangle\bigl(V(z - \cU(x)) x^{\top} Q^{\top} K + V\bigr) 
        \\  & \quad - V\cM(x)K^\top (\Delta Q)\,x\;x^{\top} Q^{\top} K + V(z - \cU(x))\,x^{\top} (\Delta Q)^{\top} K\Big],
    \end{split}
\end{align*}
and similarly,
\begin{align*}
    \begin{split}
        D_K(\gradW_\mu\cV)(z)\cdot[\Delta K] &= \alpha(x,z)\Big[\langle  Q\,x, \Delta K(z - \cU(x))\rangle\bigl(V(z - \cU(x)) x^{\top} Q^{\top} K + V\bigr) 
        \\  & \quad - V\cM(x)(\Delta K)^\top Q\,x\;x^{\top} Q^{\top} K + V(z - \cU(x))\,x^{\top} Q^{\top} (\Delta K)\Big].
    \end{split}
\end{align*}
Therefore,
\begin{align*}
    \begin{split}
        \|D_\theta (\gradW_\mu\cV_\theta)\| &\leq 3 \max\{ \|D_V (\gradW_\mu\cV_\theta)\|, \|D_Q (\gradW_\mu\cV_\theta)\|, \|D_K (\gradW_\mu\cV_\theta)\|  \}
        \\  &\leq 3 e^{2M^2 R^2} (1+2M^2R^2)^2.
    \end{split}
\end{align*}

\end{proof}
 
\begin{proposition}[Smoothness of $D_\theta\cV$]\label{prop:transformer assumption 2.3}
Adopt~\eqref{ass:compact}. There exists $L_3=L_3(M,R)$ such that $D_\theta\cV$ in~\eqref{eq:tf ic proof} satisfies~\eqref{ass:derivative}\,(c) for $x\in B_R(0)$.
\end{proposition}
 
\begin{proof}
By the commutativity of mixed partial derivatives, the bounds on $D_x(D_\theta\cV)$ and $\gradW_\mu(D_\theta\cV)$ follow from Propositions~\ref{prop:transformer assumption 2.1} and~\ref{prop:transformer assumption 2.2}. It suffices to bound $D_\theta(D_\theta\cV)$. 

From~\eqref{eq:tf DV}, we have that
\begin{align*}
    \begin{split}
        & D_K((D_V \cV)\cdot[\Delta V])\cdot [\Delta K] = \Delta V\,\cM(x)\,(\Delta K)^\top Qx, 
        \\  &D_Q((D_V \cV)\cdot[\Delta V])\cdot [\Delta Q] = \Delta V\,\cM(x)\,K^\top \Delta Qx, 
        \\  & D_V(D_V\cV) = 0.
    \end{split}
\end{align*}
Thus $\|D_\theta (D_V\cV)\|\leq 12 MR^3$.

From \eqref{eq:tf DQ},
\begin{align*}
    D_V((D_Q \cV)\cdot[\Delta Q])\cdot [\Delta V] = \Delta V\,\cM(x)\,K^\top \Delta Qx, 
\end{align*}
\begin{align*}
    \begin{split}
        D_Q((D_Q \cV)\cdot[\Delta Q])\cdot [\Delta Q'] &= V\bigg[\int\alpha(x,y)\,\langle \Delta Q ' \,x,\,K(y-\cU(x)) \rangle 
    \\  & \quad \cdot (y-\cU(x))
    (y-\cU(x))^\top\;d\mu(y)\bigg]K^\top \Delta Q x ,
    \end{split}
\end{align*}
\begin{align*}
    \begin{split}
        D_K((D_Q \cV)\cdot[\Delta Q])\cdot [\Delta K] &= V\bigg[\int\alpha(x,y)\,\langle Q \,x,\, \Delta K(y-\cU(x)) \rangle
        \\  & \quad \cdot (y-\cU(x))(y-\cU(x))^\top \d \mu(y)\bigg]K^\top \Delta Q x  + V\cM(x)(\Delta K)^\top \Delta Qx.
    \end{split}
\end{align*}
Hence, $\|D_\theta (D_Q\cV)\|\leq 8MR^3(1+2M^2R^2)$.

Similarly, 
\begin{align*}
    D_V((D_K \cV)\cdot[\Delta K])\cdot [\Delta V] = \Delta V\,\cM(x)\,(\Delta K)^\top  Qx, 
\end{align*}
\begin{align*}
    \begin{split}
        D_K((D_K \cV)\cdot[\Delta K])\cdot [\Delta K'] &= V\bigg[\int\alpha(x,y)\,\langle Q  \,x,\,\Delta K' (y-\cU(x)) \rangle
    \\  & \quad \cdot (y-\cU(x))\,(y-\cU(x))^\top\;d\mu(y)\bigg](\Delta K)^\top  Q x ,
    \end{split}
\end{align*}
\begin{align*}
    \begin{split}
        D_Q((D_K \cV)\cdot[\Delta K])\cdot [\Delta Q] &= V\bigg[\int\alpha(x,y)\,\langle \Delta Q \,x,\, K(y-\cU(x)) \rangle
        \\  & \quad \cdot (y-\cU(x))\,(y-\cU(x))^\top \d\mu(y)\bigg](\Delta K)^\top Q x  + V\cM(x)(\Delta K)^\top \Delta Qx .
    \end{split}
\end{align*}
Thus, $\|D_\theta (D_K\cV)\|\leq 8MR^3(1+2M^2R^2)$.

Combining the above estimates, we get that $\|D_\theta(D_\theta\cV)\|\leq 24MR^3(1+2M^2R^2)$.
\end{proof}

\begin{proposition}[Kernel form of $\cV$]\label{prop:transformer assumption kernel form}
Adopt~\eqref{ass:compact}. There exists $L_4=L_4(M,R)$ such that $\cV$ in~\eqref{eq:tf ic proof} satisfies \Cref{ass:kernel form V} for $m = d+1$ and all $x\in B_R(0)$.
\end{proposition}
\begin{proof}
    Let $w = (w_1,w_2,\dots,w_d,w_{d+1}) \in \RR^{d+1}$ with $w_{d+1} \neq 0$, we define
        \begin{align*}
            \cF(w) \coloneq \frac{(w_1,w_2,\dots,w_d)}{w_{d+1}},
        \end{align*}
    and for $\theta = (Q,K,V)$, define
        \begin{align*}
            F(x,y;\theta) \coloneqq \left(e^{\langle Qx,Ky\rangle}Vy, \  e^{\langle Qx,Ky\rangle}  \right) \in \RR^{d+1}.
        \end{align*}
    When $x,y \in B_R(0)$, there is a $C = C(\bA)$, such that $|F(x,y;\theta)| \leq C$ and $e^{\langle Qx,Ky\rangle} > C^{-1} > 0$. Hence, because all the derivatives of $\cF(w)$ are smooth, bounded, and dimensional free, when $|w| \leq C $ and $w_{d+1} > C^{-1}$, $\cF$ satisfies the assumptions in \Cref{ass:kernel form V} with a dimensional free constant $L_4$. Also, by analogous arguments in verifying Assumptions~\eqref{ass:V}--\eqref{ass:derivative} in this \Cref{sec:verify attention block}, $F$ also satisfies \Cref{ass:kernel form V} with an $L_4$ only depending on $R,M$ in Assumption~\eqref{ass:compact}.
\end{proof}
 
\subsection{MLP blocks}
 
For an MLP block, the velocity field is
\begin{equation}\label{eq:mlp vf}
    \cV(x;\theta)\;=\;\bW_1\,\sigma(W_2x+b),\qquad\theta=(W_1,W_2,b)\in\RR^{2d^2+d},
\end{equation}
which depends only on $x$ and $\theta$ (not on $\mu$).
 
We impose the following regularity on the activation.
 
\begin{assumption}[Activation regularity]\label{ass:activation}
The activation $\sigma:\RR^d\to\RR^d$ is twice continuously differentiable, and there exists $\alpha>0$ such that $|\sigma(x)|\leq\alpha(|x|+1)$, $\|D\sigma(x)\|\leq\alpha$, and $\|D^2\sigma(x)\|\leq\alpha$ for all $x\in\RR^d$.
\end{assumption}
 
\begin{remark}\label{rmk:ReLU} Assumption~\ref{ass:activation} is satisfied by the GELU activation used in modern LLMs. The ReLU activation $\sigma_{\rm ReLU}(t)=\max\{t,0\}$ is not $C^2$, but its piecewise-linear structure permits all results in Section~\ref{sec:main} to go through after a separate analysis that avoids the classical ODE well-posedness theory with Lipschitz coefficients. As discussed after Assumption~\ref{ass:regularity}, the assumption is needed only on the compact sets supplied by Proposition~\ref{prop:bdd MLP}; we state it globally for notational simplicity.
\end{remark}
 
\begin{proposition}[Boundedness of the MLP flow]\label{prop:bdd MLP}
Adopt~\eqref{ass:compact} for $\theta$ and $\mu_0$ and Assumption~\ref{ass:activation}. There exists $C=C(M,R,\alpha)$ such that $\supp\mu_s\subseteq B_C(0)$ for all $s\in[0,1]$.
\end{proposition}
 
\begin{proof}
The proof is similar to \Cref{prop:bdd transformer}. Setting $R_\tau=\sup\{|x|:x\in\supp\mu_\tau\}$, the linear-growth bound $|\cV(x;\theta)|\leq\alpha M(M|x|+M+1)$ together with Gr\"onwall yields the conclusion.
\end{proof}
 
\begin{proposition}[Lipschitz continuity of the MLP block]\label{prop:MLP assumption 1}
Adopt~\eqref{ass:compact} and Assumption~\ref{ass:activation}. There exist $M_0,L_0$ depending on $(M,R,\alpha)$ such that $\cV$ in~\eqref{eq:mlp vf} satisfies~\eqref{ass:V} for $x\in B_R(0)$.
\end{proposition}
 
\begin{proof}
The spatial Jacobian
\begin{equation}\label{eq:MLP Dx}
    D_x\cV(x;\theta)=W_1(D\sigma(W_2x+b))W_2
\end{equation}
satisfies $\|D_x\cV\|\leq\alpha M^2$. The directional derivatives in $\theta$ are
\begin{equation}\label{eq:MLP D theta}
    \begin{aligned}
    D_{W_1}\cV\cdot[\Delta W_1]&=(\Delta W_1) \sigma(W_2x+b),\\
    D_{W_2}\cV\cdot[\Delta W_2]&=W_1(D\sigma(W_2x+b))(\Delta W_2) x,\\
    D_b\cV\cdot[\Delta b]&=W_1(D\sigma(W_2x+b))(\Delta b),
    \end{aligned}
\end{equation}
with bounds $\alpha(MR+M+1)$, $M\alpha R$, $M\alpha$ respectively.
\end{proof}
 
\begin{proposition}[Smoothness of $D_x\cV$ for MLP]\label{prop:MLP assumption 2.1}
Adopt~\eqref{ass:compact} and Assumption~\ref{ass:activation}. There exists $L_1=L_1(M,R,\alpha)$ such that $D_x\cV$ in~\eqref{eq:mlp vf} satisfies~\eqref{ass:derivative}\,(a) for $x\in B_R(0)$.
\end{proposition}
\begin{proof}
    By \eqref{eq:MLP Dx}, we directly see that
        \begin{align*}
            \|D^2 _x \cV \cdot [h]\| \leq \alpha M^2 |W_2 h| \leq  \alpha  M^3 |h|.
        \end{align*}
    Thus, $\|D^2 _x \cV\| \leq  \alpha  M^3$. 
    
    Also, 
        \begin{align*}
            \|D_{W_1}(D_x \cV) \cdot [\Delta W_1]\| \leq  \|\Delta W_1 (D\sigma(W_2 x + b)) W_2\| \leq \|\Delta W_1\| \alpha M,
        \end{align*}
    and 
        \begin{align*}
            \|D_{W_2}(D_x \cV) \cdot [\Delta W_2]\| \leq \|W_1\| \cdot (\alpha \cdot |\Delta W_2 x |) \cdot \|W_2\| + \|W_1\| \alpha \|\Delta W_2\| \leq M\alpha(MR+1) \|\Delta W_2\|,  
        \end{align*}
    and 
        \begin{align*}
            \|D_{b}(D_x \cV) \cdot [\Delta b]\| \leq \|W_1\| \cdot (\alpha |\Delta b|) \cdot \|W_2\| \leq \alpha M^2 |\Delta b|.
        \end{align*}
    Hence, $\|D_{\theta}(D_x\cV)\|\leq\alpha M(MR+M+2)$.
\end{proof}
 
\begin{proposition}[Smoothness of $D_\theta\cV$ for MLP]\label{prop:MLP assumption 2.2}
Adopt~\eqref{ass:compact} and Assumption~\ref{ass:activation}. There exists $L_3=L_3(M,R,\alpha)$ such that $D_\theta\cV$ in~\eqref{eq:mlp vf} satisfies~\eqref{ass:derivative}\,(c) for $x\in B_R(0)$.
\end{proposition}
\begin{proof}
    Recall the terms constituting $D_\theta \cV$ in \eqref{eq:MLP D theta}. A direct computation consequently shows that
        \begin{align*}
            \begin{split}
                &D_{W_1} (D_{W_1} \cV) = 0, \quad D_{W_2} (D_{W_1} \cV \cdot [\Delta W_1]) \cdot [\Delta W_2] = \Delta W_1 (D\sigma(W_2 x + b)) \Delta W_2 x 
                \\  &D_{b} (D_{W_1} \cV \cdot [\Delta W_1]) \cdot [\Delta b] = \Delta W_1 (D\sigma(W_2 x + b)) \Delta b.
            \end{split}
        \end{align*}
    Thus, $\|D_{\theta} D_{W_1} \cV\| \leq \alpha (R+1)$.

    Similarly,
        \begin{align*}
            \begin{split}
                \|D_{W_2}(D_{W_2} \cV \cdot [\Delta W_2])\cdot [\Delta W_2 ']\| &\leq \|W_1 \| \|(D^2\sigma(W_2 x + b))\| |\Delta W_2 x| |\Delta W_2 ' x| \\    &\leq \alpha M R^2 \|\Delta W_2\| \|\Delta W_2 '\|,
            \end{split}
        \end{align*}
    and 
        \begin{align*}
            \|D_{b}(D_{W_2} \cV \cdot [\Delta W_2])\cdot [\Delta b]\| \leq \|W_1 \| \|(D^2\sigma(W_2 x + b))\| |\Delta W_2 x| |\Delta b| \leq M \alpha R \cdot \|\Delta W_2\| |\Delta b|.
        \end{align*}
    Thus, $\|D_{\theta} D_{W_2} \cV\| \leq \alpha (MR^2+MR+R)$.

    Finally,
        \begin{align*}
            \|D_{b}(D_{b} \cV \cdot [\Delta b])\cdot [\Delta b ']\| \leq \|W_1 \| \|(D^2\sigma(W_2 x + b))\| \cdot  |\Delta b| |\Delta b'| \leq M\alpha \cdot |\Delta b| |\Delta b'|.
        \end{align*}
    Hence, $\|D_{\theta} D_{b} \cV\| \leq \alpha (M+MR+1)$.

    Combining the above estimates, $\|D_{\theta} D_{\theta} \cV\| \leq C(M,R,\alpha)$. The bounds on $D_x(D_\theta\cV)$ follow from Proposition~\ref{prop:MLP assumption 2.1} by the commutativity of mixed partial derivatives. This finishes the proof for \Cref{prop:MLP assumption 2.2}.
\end{proof}

\paragraph{Conclusion.} Combining Propositions~\ref{prop:transformer assumption 1}--\ref{prop:transformer assumption 2.3} for self-attention layers and Propositions~\ref{prop:MLP assumption 1}--\ref{prop:MLP assumption 2.2} for MLP layers, every transformer block satisfies Assumptions~\eqref{ass:V} and~\eqref{ass:derivative} on the compact sets supplied by~\eqref{ass:compact} (after Propositions~\ref{prop:bdd transformer} and~\ref{prop:bdd MLP}), with constants depending only on $(M,R)$ and on the activation regularity $\alpha$, and in particular independent of the embedding dimension $d$ and the parameter dimension $p$. All results in Section~\ref{sec:main} therefore apply to standard transformer architectures with absolute constants.

\bibliographystyle{alpha}
\bibliography{references}

\end{document}